\documentclass[10pt]{article}

\usepackage[utf8]{inputenc} 
\usepackage[T1]{fontenc}

\usepackage{url}          
\usepackage{booktabs}      
\usepackage{amsfonts}       
\usepackage{nicefrac}       
\usepackage{microtype}      
\usepackage{fullpage}
\usepackage{mathtools}
\usepackage{xcolor}
\newcommand*{\colorboxed}{}
\def\colorboxed#1#{
	\colorboxedAux{#1}
}
\newcommand*{\colorboxedAux}[3]{
	\begingroup
	\colorlet{cb@saved}{.}
	\color#1{#2}
	\boxed{
		\color{cb@saved}
		#3
	}
	\endgroup
}

\usepackage{amssymb}
\usepackage{amsmath,amsthm}

\usepackage{enumitem}

\usepackage{algorithm}
\usepackage{algorithmic}

\usepackage{xcolor,colortbl}
\definecolor{DarkBlue}{RGB}{22,54,93}

\usepackage[colorlinks = true]{hyperref}        
\usepackage{cleveref}
\hypersetup{linkcolor =red, citecolor = blue}

\usepackage{natbib}
\usepackage{graphicx}
\usepackage{subfigure}
\usepackage{adjustbox}
\usepackage{amsmath}
\usepackage{amssymb}
\usepackage{mathtools}
\usepackage{amsthm}
\usepackage{threeparttable}

\theoremstyle{plain}
\newtheorem{theorem}{Theorem}[section]

\newtheorem{lemma}[theorem]{Lemma}

\theoremstyle{definition}
\newtheorem{definition}[theorem]{Definition}
\newtheorem{assumption}[theorem]{Assumption}
\theoremstyle{remark}

\usepackage[textsize=tiny]{todonotes}

\allowdisplaybreaks[4]

\usepackage{amssymb}
\usepackage{amsmath,amsthm}
\usepackage{algorithm}
\usepackage{algorithmic}

\usepackage{graphicx}
\usepackage{subcaption}
\usepackage{adjustbox}
\usepackage{amsmath}
\usepackage{amssymb}
\usepackage{mathtools}
\usepackage{amsthm}
\usepackage{threeparttable}
\usepackage{changes}

\theoremstyle{plain}
\theoremstyle{definition}

\theoremstyle{remark}

\newcommand{\argmax}{\mathop{\mathrm{argmax}}}

\newcommand{\Pb}{\mathbb{P}}
\newcommand{\Rb}{\mathbb{R}}
\newcommand{\Eb}{\mathbb{E}}
\newcommand{\Ac}{\mathcal{A}}
\newcommand{\Dc}{\mathcal{D}}
\newcommand{\Ec}{\mathcal{E}}
\newcommand{\Fc}{\mathcal{F}}
\newcommand{\Gc}{\mathcal{G}}
\newcommand{\Sc}{\mathcal{S}}
\newcommand{\Hc}{\mathcal{H}}
\newcommand{\Lc}{\mathcal{L}}
\newcommand{\Tc}{\mathcal{T}}
\newcommand{\Nc}{\mathcal{N}}
\newcommand{\Vc}{\mathcal{V}}

\newcommand{\norm}[1]{\left\lVert#1\right\rVert}

\allowdisplaybreaks

\date{}
\author
{
        Songtao Feng\thanks{\small Department of Electrical and Computer Engineering, University of Florida, FL 32603, USA; e-mail: {\tt  sfeng1@ufl.edu}}
        ,~~~Jie Fu\thanks{\small Department of Electrical and Computer Engineering, University of Florida, FL 32603, USA; e-mail: {\tt  fujie@ufl.edu}}
}

 \usepackage[utf8]{inputenc} 
 \usepackage[T1]{fontenc}    
 \usepackage{hyperref}       
 \usepackage{url}            
 \usepackage{booktabs}       
 \usepackage{amsfonts}       
 \usepackage{nicefrac}      
 \usepackage{microtype}      
 \usepackage{xcolor}        
 
 \usepackage[utf8]{inputenc} 
 \usepackage[T1]{fontenc}

 \usepackage{url}           
 \usepackage{booktabs}       
 \usepackage{amsfonts}     
 \usepackage{nicefrac}      
 \usepackage{microtype}      
 \usepackage{mathtools}
 \usepackage{xcolor}

 \usepackage{amssymb}
 \usepackage{amsmath,amsthm}

 \usepackage{enumitem}
 
 \usepackage{algorithm}
 \usepackage{algorithmic}
 
 \usepackage{xcolor,colortbl}
 \definecolor{DarkBlue}{RGB}{22,54,93}

 \usepackage{cleveref}
 
 \usepackage{natbib}
 \usepackage{graphicx}
 \usepackage{subfigure}
 \usepackage{adjustbox}
 \usepackage{mathrsfs}

 \usepackage{amsmath}
 \usepackage{amssymb}
 \usepackage{mathtools}
 \usepackage{amsthm}
 \usepackage{threeparttable}

 \theoremstyle{plain}
 
\crefname{assumption}{assumption}{assumptions}

 \usepackage[textsize=tiny]{todonotes}

 \newcommand{\URM}[1]{\left(\mathrm{\uppercase\expandafter{\romannumeral#1}}\right)}

\allowdisplaybreaks[4]

\makeatletter
\newcommand{\uset}[3][0ex]{
  \mathrel{\mathop{#3}\limits_{
    \vbox to#1{\kern -5\ex@
    \hbox{$#2$}\vss}}}}
\makeatother

\allowdisplaybreaks[4]

\title{Thompson Sampling in Online RLHF\\ with General Function Approximation}

\begin{document}
\maketitle
\begin{abstract}
Reinforcement learning from human feedback (RLHF) has achieved great empirical success in aligning large language models (LLMs) with human preference, and it is of great importance to study the statistical efficiency of RLHF algorithms from a theoretical perspective. In this work, we consider the online RLHF setting where the preference data is revealed during the learning process and study action value function approximation. We design a model-free posterior sampling algorithm for online RLHF inspired by Thompson sampling and provide its theoretical guarantee. Specifically, we adopt Bellman eluder (BE) dimension as the complexity measure of the function class and establish $O(\sqrt{T})$ regret bound for the proposed algorithm with other multiplicative factor depending on the horizon, BE dimension and the $log$-bracketing number of the function class. Further, in the analysis, we first establish the concentration-type inequality of the squared Bellman error bound based on the maximum likelihood estimator (MLE) generalization bound, which plays the crucial rules in obtaining the eluder-type regret bound and may be of independent interest.
\end{abstract}

\section{Introduction}
Reinforcement Learning with Human Feedback (RLHF, or preference-based learning) has achieved great success empirically in the fields of robotics~\cite{Jain:RLHF:2013}, games~\cite{MacGlashan:RLHF:2017,Christinano:RLHF:2017} and natural language processing~\cite{Stiennon:RLHF:2020,Ouyang:RLHF:2022}. Unlike standard reinforcement learning (RL) where the learner learns to maximize the reward, RLHF requires the learner to learn from preference feedback information in the form of trajectory-based comparison~\cite{Christinano:RLHF:2017}. The rationale is that assigning reward function for each state is challenging in many tasks, and acquiring preference data is more realistic in many tasks~\cite{Wirth:RLHF:2017}.
Despite the acclaimed effectiveness, the implementation of RLHF is far from satisfactory and often involves practical issues such as extensive parameter tuning, and preference data collection. In consequence, the resulting RLHF models typically suffer from performance degeneration issue if data collection and preference modeling are not accurate~\cite{Gao:RLHF:2023,Casper:RLHF:2023}. Therefore, it is important to understand the fundamentals of RLHF theory, which may guide future RLHF algorithmic design.

Preference-based learning has attracted growing attention recently after the tremendous empirical success in ChatGPT, and can be traced back to the study in the field of dueling bandits~\cite{Yue:bandit:2012,Saha:bandit:2021,Bengs:Bandit:2021}. A typical procedure in preference-based learning consists of training a reward model based on the pre-collected offline preference-based data, and learning a policy to optimize the learned reward model. While the aforementioned two steps are often completely separated in practice, the model learning and policy searching are sometimes accomplished at the same time especially in RLHF theory. While offline RLHF needs coverage assumption on the offline preference-based data to ensure the optimal policy can be inferred, online RLHF learn a policy by active exploration and does not need such assumption. Both offline and online RLHF has been actively investigated in RLHF theory, and in this work we focus on the online RLHF.


In this work, we aim to design new statistically efficient posterior algorithms for online RLHF with general function approximation. In particular, we seek to initiate the study of model-free Thompson sampling algorithms in RLHF and open up new problems along this direction. 
We highlight the main contributions of this work below. 

First, we focus on the \emph{online} RLHF with \emph{general function approximation}. Most existing works study the offline setting where the preference comparison data is given in advance. Beyond pure offline setting, it is also common to query comparison data during the training process. Further, we adopt the concept of Bellman eluder (BE) dimension to characterize the complexity of the function class, which is remarkably rich to cover a vast majority of tractable RL problems.

Second, we adopt \emph{posterior sampling} approach inspired by Thompson sampling (TS) under the model-free learning. It is generally hard to design computationally efficient algorithm under the general function approximation setting. However, TS algorithms are considered to be more tractable compared to UCB-based and confidence-set based algorithms in the general function approximation setting. Our work first investigates the general function approximation in posterior sampling based algorithm in RLHF setting, and we extend the definition of the general function class as well as the assumptions on realizability and completeness. These generalizations are crucial to facilitate the analysis in posterior sampling.

Third, we provide $\mathcal{O}(\sqrt{T})$ regret bound of the Thompson sampling algorithm for online RLHF with other multiplicative factors depending on the Bellman eluder dimension and the log-bracketing number of the function class. Towards the regret analysis, we establish new concentration-type inequalities with respect to the squared Bellman error based on the maximum likelihood estimator (MLE) generalization bound, which is essential for bounding the Bellman error.

Broadly speaking, our work first studies the Thompson sampling in online RLHF with action value function approximation and identifies the eluder-type regret bound. Besides the theoretical contribution, our result also provides practical insights. For example, we show that trajectories are not necessarily generated from the same policy, instead using two trajectories drawing from new and older policies actually balances exploration and exploitation. We show that the posterior sampling method based on the MLE is in some sense equivalent to confidence-set based method in that both provide the same regret guarantee using Bellman eluder dimension.

\subsection{Related works}
\if{0}
{\bf RLHF algorithms.} Proximal Policy Optimization (PPO)~\cite{Schulman:PPO:2017} is the most popular algorithm in large language models (LLMs). However, PPO suffers from instability, inefficiency, and high sensitivity to both hyperparameters~\cite{Choshen:PPO:2019} and code-level optimizations~\cite{Engstrom:PPO:2020}, which make it difficult to achieve the optimal performance in Chat-GPT4~\cite{ChatGPT4-report} and replicate its performance. Further, it requires the integration of additional components including a reward model, a value network (critic), and a reference model, potentially as large as the aligned LLM~\cite{Ouyang:ChatGPT:2022,Touvron:2023}. To resolve these aforementioned limitations, researchers have explored alternative strategies for LLM alignment. One approach is the reward-ranked finetuning (RAFT) that iteratively finetunes the model on the best outputs from a set of generated responses to maximize reward~\cite{dong2023raft,2023arXiv230405302Y,Touvron:2023,2023arXiv230808998G}. Another line of research builds upon the KL-regularized formulation~\cite{2023arXiv230518290R,2023arXiv230602231Z,2023arXiv230916240W,2023arXiv230906657L,2023arXiv231010505L}. Notably, Direct Preference Optimization (DPO)~\cite{2023arXiv230518290R} has been widely adopted as a promising alternative to PPO with notable stability and competitive performance. Similar to DPO, some other works also aims to optimize the LLM directly from the preference data~\cite{2023arXiv230510425Z,2023arXiv231012036G} but the theoretical analysis is still lacking.
\fi

{\bf Dueling bandits with preference feedback.} Dueling bandits are simplified settings within the RLHF framework where the learner proposes two actions and only provide preference feedback. Numerous works have studied the dueling bandits~\cite{Yue:bandit:2012,Saha:bandit:2021,Bengs:Bandit:2021} and contextual dueling bandits~\cite{Dudik:bandit:2015,Saha:bandit:2022,Yue:bandit:2024}.

{\bf RLHF theory.}  RLHF has been first investigated in the tabular setting~\cite{Yichong:RLHF:2020,Ellen:RLHF:2019,Saha:RLHF:2023}. Beyond tabular setting, \cite{Xiaoyu:RLHF:2022} proposed the first confidence-set based algorithm for the online RLHF with general function approximation. Later \cite{Yuanhao:RLHF:2023} proposed a framework that transforms exiting sample-efficient RL algorithms to efficient online RLHF algorithms. Additionally, \cite{Wenhao:RLHF:2023,Zihao:RLHF:2023} investigated general function approximation for offline RLHF. Additionally, \cite{Xiong:RLHF:2023} introduced hybrid RLHF. The most relevant work~\cite{Runzhe:RLHF:2023} developed algorithms for online RLHF for both linear function approximation and nonlinear function approximation. Despite both their work and ours leverage the TS idea, they propose a model-based algorithm while our algorithm is model-free. Further, the nonlinear function approximation in their work and the general function approximation proposed in our work is fundamentally different.

{\bf Thompson sampling (TS).} TS was first proposed in \cite{Thompson:TS:1933} for stochastic bandit problems. Thanks to the strong empirical performance~\cite{Osband:TS:2018,IshFaq:TS:2023}, TS has been investigated in a wide range of theoretical learning problems including bandit problems~\cite{Agrawal:TS:2012,Agrawal:TS:2017,Agrawal:TS:2013,SiweiWang:TS:2018} and RL theory~\cite{Russo:TS:2019,Zanette:TS:2019,IshFaq:TS:2021,ZhihanXiong:TS:2022,Wenhao:RLHF:2023,Zihao:RLHF:2023}.

\section{Preliminary}
{\bf Notation:} For two integers $M$ and $N$ where $M\leq N$, we denote $[M:N]=\{M,M+1,\ldots,N\}$ and $[N]=\{1,2,\ldots,N\}$. For two real numbers $a$ and $b$ where $a\leq b$, we denote $[a,b]=\{x\mid a\leq x\leq b\}$. For a set $\mathcal{S}$, we use $\Delta(\mathcal{S})$ to denote the set of distributions over $\mathcal{S}$.

{\bf Markov decision processes (MDP).} Consider an episodic Markov decision process denoted by $\mathcal{M}=(H,\mathcal{S},\mathcal{A},r,P,s_1)$, where horizon $H$ is the length of the episode, $\mathcal{S}$ is the state space, $\mathcal{A}$ is the action space, $r=\{r_h\}_{h=1}^H$ and $P=\{P_h\}_{h=1}^{H-1}$ are the collection of reward functions and transition kernels, and $s_1$ is a fixed initial state. Here $r_h:\mathcal{S}\times\mathcal{A}\to [0,1]$ is the $h$-th step reward function, and $P_h:\mathcal{S}\times\mathcal{A}\to\Delta(\mathcal{S})$ is the $h$-th step transition kernel. We assume each episode ends in a fixed final state $s_{H+1}$, and $P_{H}(s_{H+1}|s,a)=1$ for all $(s,a)\in\mathcal{S}\times\mathcal{A}$.

The sequential interaction between the learner and the environment  is captured by the MDP $\mathcal{M}$. The interaction proceeds in $H$ rounds. At each step $h\in[H]$, the learner observes   state $s_h\in\mathcal{S}$, takes   action $a_h\in\mathcal{A}$, and the MDP evolves into the next state $s_{h+1}\sim P_h(\cdot|s_h,a_h)$. The episode ends when the final state $s_{H+1}$ is reached. While in standard RL the learner has access to the per-step reward signal, in RLHF the preference feedback is revealed instead. We will elaborate it later. 

A policy $\pi=\{\pi_h\}_{h=1}^H$ is characterized by $\pi_h:\mathcal{S}\to\Delta(\mathcal{A})$, which specifies the probability of selecting actions at each step $h\in[H]$. Given a policy $\pi$, the value function at step $h$ is defined as $V_h^\pi(s)=\mathbb{E}_{\tau_{h:H}\sim\pi}[\sum_{h'=h}^Hr(s_{h'},a_{h'})|s_h=s]$, where we use $\tau_{h:H}$ to denote the trajectory $\tau_{h:H}=(s_h,a_h,\ldots,s_{H+1})$ induced by policy $\pi$ starting from certain state $s_h$. The action value function $Q_h^\pi(s,a)$ is defined similarly. The optimal policy $\pi^*$ is the policy that maximizes the expected reward $\pi^*=\argmax_\pi V_1^\pi(s_1)$. For simplicity, we use $V^*$ to denote $V^{\pi^*}$ and $Q^*$ to denote $Q^{\pi^*}$ respectively.

{\bf Reinforcement learning with human feedback (RLHF).} In the mathematical framework of online RLHF, the learner needs to select two action sequences at the beginning of episodes $t$, and generate two trajectories $\tau_t^0$ and $\tau_t^1$, where $\tau_t^i=(s_{t,1}^i,a_{t,1}^i,\ldots,s_{t,H+1}^i)$ for $i\in\{0,1\}$. Then a preference feedback $o_t\in\{0,1\}$ between $\tau_t^0$ and $\tau_t^1$ is revealed. Here $o_t=0$ indicates $\tau_t^0$ is preferred over $\tau_t^1$, and similarly $o_t=1$ indicates $\tau_t^1$ is preferred over $\tau^0$. Often the policies for generating trajectories $\tau_t^0$ and $\tau_t^1$ are selected at the beginning of the each episode, and we define $\pi_t^0$ and $\pi_t^1$ as the policies that generate trajectories $\tau_t^0$ and $\tau_t^1$ respectively. We adopt the following preference model. 


\begin{definition}[Preference model \cite{Wenhao_Zhan:RLHF:2024}]\label{defn:pref-model}
Given a pair of trajectories $(\tau^0,\tau^1)$, the preference feedback $o\in\{0,1\}$ is sampled from the Bernoulli distribution
\begin{align*}
\Pb(o=0|\tau^0,\tau^1)&=\Pb(\mbox{$\tau^0$ is preferred to $\tau^1$}|r) =\Phi(r(\tau^0)-r(\tau^1)),
\end{align*}
where $r(\tau^i)=\sum_{h=1}^H=r(s_{h}^i,a_h^i)$ for $i\in\{0,1\}$ is the trajectory reward, and $\Phi:\Rb\to[0,1]$ is a monotonically increasing link function.
\end{definition}

\begin{assumption}[\cite{Runzhe:RLHF:2023}]\label{assump:link-fcn}
Assume link function $\Phi$ is differentiable and there exists constants $\kappa , \bar{\kappa}$ such that $\kappa^{-1}\leq \Phi'(x)\leq \bar{\kappa}^{-1}$ for any $x\in[-H,H]$.
\end{assumption}

Constant $\kappa$ and $\bar{\kappa}$ characterizes how difficulty of estimating the reward signal from preference feedback. It is worth noting that $\kappa$ and $\bar{\kappa}$ play different roles as our theoretical result shows that the bound depends polynomially on $\kappa$ but logarithmically on $\bar{\kappa}$.

When we choose sigmoid function $\Phi(x)=1/(1+\exp(-x))$ as the link function, the preference model becomes the well-known Bradley-Terry-Luce (BTL) model~\cite{Bradley:BTL:1952}. In this case, $\kappa=\bar{\kappa}=2$.

Under the online RLHF framework, we aim to minimize the following regret 
\begin{align*}
\mathrm{Regret}(T)=\sum_{t=1}^T\left(2V_1^*(s_1)-V_1^{\pi_t^0}(s_1)-V_1^{\pi_t^1}(s_1)\right),
\end{align*}
which essentially quantifies performance gap between optimal policy $\pi^*$ and agent-chosen policy $\pi_t^i$. Besides the regret, we also consider the sample complexity. Specifically, we aim to design algorithms for online RLHF with general function approximation that can find $\epsilon$-optimal policy efficiently. 


\subsection{Function approximation}
We are interested in the online RLHF with general function approximation. There are different ways to introduce general function approximation, and we consider action-value function approximation. 

Given a function class $\Fc=\Fc_1\times\cdots\times\Fc_{H}$, where $\Fc_h\subseteq(\Sc\times\Ac\to[0,H-h+1])$ gives a set of candidate functions to approximate optimal action-value function $Q_h^*$. To ensure efficient algorithms for RL with general function approximation, The following assumptions are commonly adopted in literature~\cite{Chi_Jin:RL_eluder:2021}.

\begin{assumption}[Realizability]\label{assump:realizability}
$Q_h^*\in\Fc_h$ for all $h\in[H]$.
\end{assumption}
Realizability assumption asserts that function class $\Fc$ is rich enough to include the optimal $Q^*$.

The Bellman operator $\Tc_h$ at step $h$ is defined by the well-known Bellman equation
\begin{align*}
Q_h^*(s,a)&=(\Tc_h Q_{h+1}^*)(s,a) =r_{h}(s,a)+\Eb_{s'\sim P_h(\cdot|s,a)}[\max_{a'}Q_{h+1}^*(s',a')],
\end{align*}
for all valid $(s,a,h)$. Define $\Tc_h\Fc_{h+1}=\{\Tc_h f_{h+1}:f_{h+1}\in\Fc_{h+1}\}$.
\begin{assumption}[Completeness]\label{assump:completeness}
$\Tc_h\Fc_{h+1}\subseteq \Fc_h$.
\end{assumption}
Completeness assumption requires that the function class is closed under the Bellman operator.

When function class $\Fc$ has finite elements, the cardinality $|\Fc|$ naturally serves as the ``size" measure of function class $\Fc$. When addressing general function approximation with infinitely many elements, we need other notions. In this work, we use the notion of bracketing number~\cite{Geer2000EmpiricalPI}.

\begin{definition}[Bracketing number]
Consider a function class $\Fc\subseteq \mathcal{X}\to \Rb$. Given two functions $l,u:\mathcal{X}\to \Rb$, the bracket $[l,u]$ is defined as the set of functions $f\in\Fc$ with $l(x)\leq f(x)\leq u(x)$ for all $x\in\mathcal{X}$. It is called an $\omega$-bracket if $\norm{l-u}\leq \omega$. The bracketing number of $\Fc$ with respect to the metric $\norm{\cdot}$, denoted by $\Nc_{[]}(\omega,\Fc,\norm{\cdot})$, is the minimum number of $\omega$-brackets needed to cover $\Fc$.
\end{definition}
Bracketing number has shown to be small in many practical scenarios~\cite{Geer2000EmpiricalPI}. When $\Fc$ is finite, the bracketing number is simply bounded by its size. When $\Fc$ is a $d$-dimensional linear function class, the logarithm of its bracketing number is bounded by $d$ up to logorithmic factors. Further, the well-known covering number $\Nc(\epsilon,\Fc,\norm{\cdot})$ is upper bounded by the bracketing number $\Nc_{[]}(2\epsilon,\Fc,\norm{\cdot})$ for any $\epsilon>0$. In the sequel, we sometimes write $\Nc_{[]}(\omega,\Fc,\norm{\cdot})$ as $\Nc_{[]}(\Fc)$ for simplicity.

When invoking the bracketing number to our RLHF with general function approximation scenario, we apply the bracketing number to each function class $\Fc_h$ for $h\in[H]$ and write $\Nc_{[]}(\epsilon,\Fc,\norm{\cdot}):=\sup_{h\in[H]}\Nc_{[]}(\epsilon,\Fc_h,\norm{\cdot})$ for notational convenience.

\subsection{Bellman eluder dimension}
The notion of Bellman eluder (BE) dimension was first proposed in~\cite{Chi_Jin:RL_eluder:2021} to capture rich classes of RL problems. For completeness, we provide formal definitions below.

We start with the distributional eluder dimension.
\begin{definition}[Distributional eluder Dimension]\label{def:DE-dim}
Let $\epsilon \geq 0$ and $\{\nu_i\}_{i=1}^n\subseteq \Delta(\Sc\times\Ac)$ be a sequence of probability distributions.
\begin{itemize}
\item A distribution  $\mu\in\Delta(\Sc\times\Ac)$ is $\epsilon$-dependent on $\{\nu_1,\ldots,\nu_n\}$ with respect to $\Fc$ if any $f\in\Fc$ satisfying $\sqrt{\sum_i (\Eb_{\nu_i}f)^2}\leq \epsilon$ also satisfies $|\Eb_{\mu}f|\leq  \epsilon$.
\item A distribution $\mu$ is $\epsilon$-independent on $\{\nu_1,\ldots,\nu_n\}$ with respect to $\Fc$ if $\mu$ is not $\epsilon$-dependent on $\{\nu_1,\ldots,\nu_n\}$.
\item The distributional eluder dimension $\dim_E(\Fc,\Pi,\epsilon)$ of a function class $\Fc$ and distribution class $\Pi$ is the length of the longest sequence of elements in $\Pi$ such that, for some $\epsilon'\geq \epsilon$, every element is $\epsilon'$-independent of its predecessors.
\end{itemize}
\end{definition}

When we consider the function class of Bellman residuals, the distributional eluder dimension becomes Bellman eluder dimension when maximizing over all steps.
\begin{definition}[Bellman eluder (BE) Dimension]\label{def:BE-dim}
Let $\Fc_h-\Tc_h\Fc_{h+1}=\{f_h-\Tc_hf_{h+1}:f\in\Fc\}$ be the set of Bellman residuals induced by $\Fc$ at step $h$, and $\Pi=\{\Pi_h\}_{h=1}^H$ be a collection of $H$ probability measure families over $\Sc\times\Ac$. The $\epsilon$-Bellman eluder of $\Fc$ with respecti $\Pi$ is defined as
\begin{align*}
\dim_{\mathrm{BE}}(\Fc,\Pi,\epsilon)=\sup_{h\in[H]}\dim_{\mathrm{BE}}(\Fc_h-\Tc_h\Fc_{h+1},\Pi_h,\epsilon).
\end{align*}
\end{definition}

Besides the function class $\Fc$ and error $\epsilon$, the Bellman eluder dimension also depends on the distribution family $\Pi$. In this work, we consider the following two choices.
\begin{itemize}
\item $\Pi_\Fc=\{\Pi_{\Fc,h}\}_{h\in[H]}$: $\Pi_{\Fc,h}$ denotes the collection of all probability measures over $\Sc\times\Ac$ at step $h$ that are generated by executing the greedy policy $\pi_f$.
\item $\Pi_\Delta=\{\Pi_{\Delta,h}\}_{h\in[H]}$: $\Pi_{\Delta,h}=\{\delta_{(s,a)}(\cdot):s\in\Sc,a\in\Ac\}$, the collections of probability measures that put measure 1 on a single state-action pair.
\end{itemize}

\section{Posterior sampling}
In this section, we first derive the posterior sampling expression, and then generalizing the function approximation and related assumptions in posterior sampling regime.

In posterior sampling, the underlying model is not necessarily a fixed model. Instead, it could be sampled from certain distribution. Suppose $f^*\in\Fc$ is the underlying model following distribution $p(f^*)$, then we write action-value function as $Q_h^*(s,a):=Q_{h;f^*}(s,a)=\Eb_{f\sim p(f^*)}[Q_{h;f}(s,a)]$ for all valid $h,s,a$. In the analysis, we use simplified notation $\Eb_{f^*}[g(f^*)]:=\Eb_{f\sim p(f^*)}[g(f)]$ for any function $g$. The value function is defined similarly. 

We first reformulate the regret goal in Bayesian RL framework. In the Bayesian setting, $f^*$ is sampled from some known prior distribution $p_0$, and the goal is minimize the Bayesian regret
\begin{align*}
\mathrm{Regret}(T)=\Eb\left[\sum_{t=1}^T\left(2V_1^*(s_1)-V_1^{\pi_t^0}(s_1)-V_1^{\pi_t^1}(s_1)\right)\right],
\end{align*}
where the expectation is taken with respect to the prior distribution over $f^*$. 


We assume that the prior over $\Fc$ has the form $p_0$. Let $\Dc^t=\{(\tau_0^s,\tau_1^s,o^s)\}_{s=1}^t$ be the comparison data observed in $t$ episodes. Then the likelihood of $\Dc^t$ given a value function $f\in\Fc$ is defined as
\begin{align*}
p(\Dc^t|f)=&\prod_{s=1}^t\Big(o^s\Phi(r(\tau_1^s)-r(\tau_0^s)) +(1-o^s)\Phi(r(\tau_0^s)-r(\tau_1^s))\Big).
\end{align*}
Here $r(\tau)=\sum_{h=1}^H (f_h-P_hf_{h+1})(s_h,a_h)$ where $\tau=(s_1,a_1,\ldots,s_{h+1})$ is the trajectory and $P_h$ is the underlying transition kernel. In the main context, we assume $P_h$ is available for simplicity. In practice, we can form the estimator at round $t$ as $\widehat{P}_h^t(s'|s,a)=\frac{\sum_{\tau\in\Dc}\mathbf{1}\{(s,a,s')\in\tau\}}{\max\{\sum_{\tau\in\Dc}\mathbf{1}\{(s,a)\in\tau\},1\}}$, where $\Dc$ consists of all trajectories up to round $t$, $(s,a,s')\in\tau$ means trajectory in $\tau=(s_1,a_1,\ldots,s_{H+1})$, $(s_h,a_h,s_{h+1})=(s,a,s')$ and $(s,a)\in\tau$ is defined similarly. We remark that $\norm{(\widehat{P}_h^t-P_h)f_h}_\infty<\widetilde{\mathcal{O}}(\sqrt{\frac{1}{t}})$ with other logarithmic multiplicative factor scaling in $\log H\cdot\log\Nc_{[]}(\Fc)$. As a result, replacing $P_h$ by $\widehat{P}_h$ in the computation of $r(\tau)$ does not incur any leading term in the regret bound.

Combining the prior and likelihood above, we obtain the posterior after $t$ episodes
\begin{align}
p(f|\Dc^{t})\propto p_0(f) &\prod_{s=1}^t\Big(o^s\Phi(r(\tau_1^s)-r(\tau_0^s)) +(1-o^s)\Phi(r(\tau_0^s)-r(\tau_1^s))\Big), \label{eqn:post}
\end{align}
which ignore factors $1/p(\Dc^t)$ as it is a constant under the underlying MDP.

\subsection{Function class completion}
Different from the UCB-based or the confidence-set based approach, posterior sampling may require generalizing the definition function approximation and the assumptions on realizability and completeness.

In standard RL problems, the underlying model $f^*$ is fixed  while the underlying model could draw from certain distribution. The function class completion is needed for posterior sampling method to ensure the posterior sampled model following certain distribution (no longer fixed) still lies in the generalized function class. 

We begin with generalizing the function class $\Fc$ in posterior sampling scenario. The necessity lies in that we expect the generalized function class is complete under posterior sampling procedure.

\begin{definition}[Generalized function approximation class]\label{defn:gen-fcn-approx}
Given a function class $\Fc$ and prior set $\Pi$ consisting of all possible prior distribution. we complete the function class, denoted as $\Gc$, so that any possible posterior lies in $\Gc$.
\end{definition}
Particularly, we have that any singleton $f\in\Fc\subseteq \Gc$, and any prior $\Pi\subseteq \Gc$. This step is to ensure that generalized class is complete under the posterior sampling procedure, therefore $\Gc$ is rich enough to include all possible functions (possibly stochastic) during any algorithm runs.

we define generalized realizability for posterior sampling below.
\begin{assumption}[Generalized realizability]\label{assump:gen-realizability}
Let $f^*_h$ be the underlying function class of step $h$, possibly stochastic, which follows distribution $p(f^*)$. Then, $p(f^*_h)\subseteq \Delta(\Fc_h)$ for all $h\in[H]$.
\end{assumption}
The generalized realizability assumption in posterior sampling essentially requires that for any possible model $f_h\in p(f^*)$ (with positive probability), $f_h\in\Fc_h$. We comment that this assumption could be replaced by that $Q_h^*\in\Gc_h$ for all $h\in[H]$.

In posterior sampling, the Bellman operator $\Tc_h$ at step $h$ is defined similarly by the Bellman equation $Q_h^*(s,a)=(\Tc_hQ_{h+1}^*)(s,a)$, where the difference lies in that $Q_h^*(s,a)$ is averaging over $p(f^*)$. 

\begin{assumption}[Generalized completeness]\label{assump:gen-completeness}
$\Tc_h\Gc_{h+1}\subseteq \Gc_h$.
\end{assumption}
Our generalized completeness assumption is essentially the same as standard completeness assumption except that generalized function approximation $\Gc_h$ is considered.

Next, we show upper bounds for the BE dimension of the completed linear model class and the function class with finite number of elements. 

In the linear model where the action value function $f_h\in\Fc_h$ has form $f_h(s,a)=\langle\phi(s,a),\theta\rangle$ with a known feature $\phi(\cdot,\cdot)$ of dimension $\Tilde{d}$, the completed model class $\Gc_\Fc$ is still linear. Therefore the BE dimension for linear function class is bounded by $\widetilde{\mathcal{O}}(\Tilde{d})$ up to some logarithmic factor.

For the finite function class with cardinality $|\Fc|$, assume the prior distribution is uniform (the prior we adopt in the algorithm) and our algorithm ends within $T$ rounds. Its covering number is upper bounded by $T\log 2$ by considering $2^T$ possible outcomes of the dataset. Hence in general the finite function class may end up with linear regret (proportional to the log-covering number). In practice, the target function $f$ is always approximated with certain precision, and we may further discretize the continuous function space based on the precision. 

We leave as our future work whether known tractable classes after completion are still tractable. It is also worth exploring the case when generalized completeness is violated, and we leave this as our future work.

\section{Model-free Thompson sampling algorithm}
In this section, we propose our model-free posterior sampling algorithm for online RLHF with general function approximation. The posterior sampling is inspired by the well-known TS method, which is often more tractable than the UCB-based and the confidence-set based algorithms. 


\subsection{Algorithm}

{\bf Planning oracle.} $\arg\max_\pi f$ is a planning oracle which gives the greedy policy with respect to $f$ for any function $f\in\Gc$. The planning oracle is standard in the general function approximation scenario. %

\begin{algorithm}[ht]
\begin{algorithmic}[1]
\caption{Model-free Thompson sampling for online RLHF with general function approximation}\label{alg:main}
\STATE{Initialization: prior distribution $p_0$, uniform policy $\pi_0^0$, dataset $\mathcal{D}^0=\varnothing$.}
\FOR{$t\leftarrow 1,\ldots,T$}
\STATE{Draw $Q$ function $f^t\sim p(\cdot|S_{t-1})$ according to (\ref{eqn:post}).}
\STATE{Compute greedy policy $\pi_t^0=\arg\max_{\pi}f$, and set $\pi^t_1=\pi^{t-1}_0$.}
\STATE{Employ $\pi^t_0,\pi^t_1$ to obtain a trajectory comparison $(\tau^t_0,\tau^t_1,o^t)$, and set $\mathcal{D}^t=\mathcal{D}^{t-1}\cup (\tau^t_0,\tau^t_1,o^t)$.}
\ENDFOR
\end{algorithmic}
\end{algorithm}

{\bf Algorithm description.} Our algorithm is presented in Algorithm~\ref{alg:main}. Our proposed algorithm proceeds in $T$ rounds. At the beginning of episode $t$, we first compute the posterior distribution on $f^t=(f_1^t,\ldots,f_H^t)\in\Fc$ by (\ref{eqn:post}) and sample a $f^t$ based on the posterior. Then, we use a standard planning oracle $\arg\max_\pi f$ to obtain the greedy policy $\pi^t_1$. The comparator policy $\pi^t_1$ is simply set to be the policy $\pi^t_0$ obtained from the previous episode. Employing the two policies gives a new trajectory comparison data, and we use such data to update the dataset $\Dc^t$ at the end of the round.

{\bf Computation.} The bottleneck of our algorithm lies in the computation of posterior distribution (Line 3). Existing works showed that using approximation posterior sampling can achieve similar performances but without theoretical guarantee~\cite{Osband:TS:2023}. 

\subsection{Theoretical guarantee}
We provide the theoretical guarantee for Algorithm~\ref{alg:main}, which holds under Assumption~\ref{assump:link-fcn}, Assumption~\ref{assump:gen-realizability} and Assumption~\ref{assump:gen-completeness}.

\begin{theorem}[High probability regret bound of Algorithm~\ref{alg:main}]\label{thm:main}
Under Assumption~\ref{assump:link-fcn}, Assumption~\ref{assump:gen-realizability} and Assumption~\ref{assump:gen-completeness}, for any $\delta\in(0,1]$ and $T\in\mathbb{N}$, with probability at least $1-\delta$, for all $t\in[T]$, it holds that
\begin{align*}
\mathrm{Regret}(t)\leq \mathcal{O}\left(H\sqrt{dt\beta_\Gc}\right),
\end{align*}
where $d=\min_{\Pi\in\{\Pi_\Delta,\Pi_\Gc\}}\dim_{\mathrm{BE}}(\Gc_\Fc,\Pi,\sqrt{1/T})$ is the BE dimension, $\Gc_\Fc$ is the function class after completing $\Fc$, and $\beta_\Gc=\sup_{h\in[H]}98\kappa^2\log(2HN_{[]}(\bar{\kappa}(2T)^{-1},\Gc_h,\norm{\cdot}_\infty)/\delta)$.
\end{theorem}
Note that $\log N_{[]}((c/T),\Fc_h,\norm{\cdot}_\infty)\leq \mathcal{O}(\log T)$ and therefore $\beta_\Gc(T)\leq \widetilde{O}(\log T)$. 
Theorem~\ref{thm:main} asserts that, the RLHF problem is tractable if the completed function class $\Gc_\Fc$ has low BE dimension. I.e., Algorithm~\ref{alg:main} achieves $\widetilde{O}(\sqrt{T})$ regret, with multiplicative factors depending on the horizon $H$, the BE dimension $d$, and the log-bracketing number. 

To the best of our knowledge, this is the first eluder-type regret bound in the RLHF setting. While the eluder-type regret bound was first establish for the confidence-set based algorithm~\cite{Chi_Jin:RL_eluder:2021}, we obtain such bound for the Thompson sampling Algorithms. In Thompson sampling, the posterior distribution is essentially an MLE. Surprisingly, we are able to build the connection between the MLE generalization bound and the squared Bellman error, which is the key to establish the regret bound regarding Bellman eluder dimension. We will elaborate the key ideas in Section~\ref{sec:key-idea}.

We observe that the regret bound has the same form as that in the standard RL setting~\cite{Chi_Jin:RL_eluder:2021}. However, as will be explained later, the analysis is fundamentally different due to the posterior sampling approach adopted in our algorithm. Further, our result can be easily extended to the Thompson sampling algorithm for the standard RL.

Besides the high probability regret bound, we also provide the expected regret bound below.
\begin{theorem}[Expected regret of Algorithm~\ref{alg:main}]\label{thm:main-2}
Under the same assumption as in Theorem~\ref{thm:main}, for all $t\in[T]$, it holds that
\begin{align*}
\mathrm{Regret}(t)\leq \mathcal{O}\left(H\sqrt{dt\beta'_\Gc}+\frac{1}{T}\sqrt{td}\right)
\end{align*}
where $d=\min_{\Pi\in\{\Pi_\Delta,\Pi_\Gc\}}\dim_{\mathrm{BE}}(\Gc_\Fc,\Pi,\sqrt{1/T})$ is the BE dimension, $\Gc_\Fc$ is the function class after completing $\Fc$, and $\beta'_\Gc=\sup_{h\in[H]}98\kappa^2\log(2HTN_{[]}(\bar{\kappa}(2T)^{-1},\Gc_h,\norm{\cdot}_\infty))$.
\end{theorem}
The first term is incurred by the high probability event and the second term is by the low probability event. By the law of total probability and selecting confidence level parameter $\delta$ lead to the result in Theorem~\ref{thm:main-2}.

\subsection{Key ideas in proving Theorem~\ref{thm:main}}\label{sec:key-idea}

In this section, we use $\widehat{f}^t$ (instead of $f^t$ in the algorithm) to denote the posterior to emphasize it is an estimator.

Our first observation is that the regret expression can be upper bounded by the cumulative Bellman error, i.e.,
\begin{align}
\mathrm{Regret}(t)\leq 2\sum_{h=1}^H\sum_{t=1}^T\underset{(s_h,a_h)\sim\pi_0^t}{\Eb}\left[(f_h^t-\Tc_hf_{h+1}^t)(s_h,a_h)\right]. \label{eqn:new++}
\end{align}

The follow lemma concerning the distributional eluder dimension~\cite{Chi_Jin:RL_eluder:2021} is the key to build connection between $\sum_{i=1}^{t-1}\underset{(s_h,a_h)\sim\pi_i}{\Eb}[(f_h^t-\Tc_h f_{h+1}^t)(s_h,a_h)]$ and $\sum_{i=1}^{t}\underset{(s_h,a_h)\sim\pi_i}{\Eb}[(f_h^i-\Tc_h f_{h+1}^i)(s_h,a_h)]^2$.

\begin{lemma}[Simplification of Lemma~\ref{lemma:dist-eluder}]\label{lemma:dist-eluder-sim}
Given a function class $\Phi$ defined on $\mathcal{X}$ with $|\phi(x)|\leq 1$ for all $(g,x)\in\Phi\times\mathcal{X}$, and a family of probability measure $\Pi$ over $\mathcal{X}$. Suppose sequence $\{\phi_k\}_{k=1}^K\subseteq \Phi$ and $\{\mu_k\}_{k=1}^K\subseteq \Pi$ satisfy that for all $k\in[K]$, $\sum_{t=1}^{k-1}(\Eb_{\mu_t}[\phi_k])^2\leq \beta$. Then for all $k\in[K]$ and $\omega >0$,
\begin{align*}
\sum_{t=1}^k|\Eb_{\mu_t}[\phi_t]|\leq \mathcal{O}\left(\sqrt{\dim_{\mathrm{DE}}(\Phi,\Pi,1/k)\beta k}\right).
\end{align*}
\end{lemma}

Inspired by the above lemma, we aim to show the concentration-type inequality $\sum_{i=1}^{t}\underset{(s_h,a_h)\sim\pi_i}{\Eb}[(f_h^i-\Tc_h f_{h+1}^i)(s_h,a_h)]^2<\mathcal{O}(\beta)$. In the work of \cite{Chi_Jin:RL_eluder:2021}, such inequality is established by concentration inequality and its special algorithm design. This work features the posterior sampling which is fundamentally different, we therefore adopt a totally different approach. 

Our approach first leverage the fact that our posterior distribution is the maximum likelihood estimator (MLE). Based on the MLE generalization bound, we are able to construct the following confidence set 
\begin{align}
\Vc_1^\Fc=\left\{(f_1,\ldots,f_H)\in\Fc:\sum_{t=1}^T\underset{\Hc_{t-1}}{\Eb}\left[\underset{f_h}{\Eb}\left[(f_h(x_t)-\widehat{f}_h(x_t))^2\middle|\Hc_{t-1}\right]\right]\leq 4\beta_{\Fc_h}(T),\forall h\right\},  \label{eqn:version-1}
\end{align}
where $\beta_{\Fc_h}(t)=98\kappa^2\log(2HN_{[]}(\bar{\kappa}(2t)^{-1},\Fc_h,\norm{\cdot}_\infty)/\delta)$, $\widehat{f}_h$ is the maximum likelihood estimator at step $h$, $f_h$ is sampled from the posterior of $f_{h}^*$ conditioning on $\Hc_{t-1}$ in the inner expectation, and $x_t=(s_h^t,a_h^t)$ is the $t$-th round state-action pair at step $h$. 
\begin{lemma}\label{lemma:key-idea}
Define $\Vc_1^\Fc$ in (\ref{eqn:version-1}). With probability at least $1-\delta$, it holds that (1) $f^*\in\Vc_1^\Fc$; and (2) for any $f\in\Vc_1^\Fc$, $\sum_{i=1}^{t-1}\underset{x\sim\pi_i^0}{\Eb}[(f_h(x)-(\Tc_h f_{h+1}^t)(x))^2]\leq \mathcal{O}\left(\beta_{\Fc_h}(t)+\beta_{\Fc_{h+1}}(t)\right)$.
\end{lemma}
The proof is deferred to Appendix~\ref{app:version-space}. We comment that Lemma~\ref{lemma:key-idea} is stronger than its counterpart for the confidence-set based approach for standard RL with general function approximation (Lemma~39 in \cite{Chi_Jin:RL_eluder:2021}). Specifically, the inequality in \cite{Chi_Jin:RL_eluder:2021} only holds for the function actually employed in each round, while our inequality holds for any function within the confidence set. 


Combining Lemma~\ref{lemma:key-idea} and Lemma~\ref{lemma:dist-eluder-sim} gives an upper bound for the term $\sum_{t=1}^T\underset{(s_h,a_h)\sim\pi_0^t}{\Eb}\left[(f_h^t-\Tc_hf_{h+1}^t)(s_h,a_h)\right]$, and the regret (\ref{eqn:new++}) can thus be analyzed.

\subsection{Proof sketch of Theorem~\ref{thm:main}}
We provide proof sketch in this section and the complete proof is deferred to Appendix~\ref{app:thm_proof}.

{\bf Step I: Simplifying the regret term.} The main contribution is to show $\Eb_{\Hc^{t-1}}\Eb_{f^*}[V_{1;f^t}^{\pi_0^t}|\Hc^{t-1}]=\Eb_{\Hc^{t-1}}\Eb_{f^*}[V_{1;f^*}^{\pi^*}|\Hc^{t-1}]$, where we define $V_{1;f}^\pi$ as the value function under policy $\pi$ with respect to a model $f\in\Fc$. Define $\mathcal{H}^{t}:=\{(\tau_0^s,\tau_1^s,o^s)\}_{s=1}^t$ as the history up to round $t$.

In fact, since the posterior $p(f|\Dc^t)$ is the product of the prior and the likelihood, thus $f^t$ and $f^*$ are identically distributed given history $\Hc^{t-1}$. Based on the above observation and the fact that $\pi_0^t$ and $\pi^*$ are optimal policies of $f^t$ and $f^*$, we conclude that $V_{1;f^t}^{\pi_0^t}$ and $V_{1;f^*}^{\pi^*}$ are identically distributed given history $\Hc^{t-1}$, and we have $\Eb_{\Hc^{t-1}}\Eb_{f^*}[V_{1;f^t}^{\pi_0^t}|\Hc^{t-1}]=\Eb_{\Hc^{t-1}}\Eb_{f^*}[V_{1;f^*}^{\pi^*}|\Hc^{t-1}]$.

The Bayesian regret can be bounded as follows
\begin{align*}
\mathrm{Regret}(T)&=\sum_{t=1}^T\left(2V_1^*-V_1^{\pi_0^t}-V_1^{\pi_1^t}\right)(s_1)  
\leq 2\sum_{t=0}^T\left(V_1^*-V_1^{\pi_0^t}\right)(s_1) \\
&=2\sum_{t=0}^T\underset{f^*}{\Eb}\left(V_{1;f^*}^*-V_{1;f^*}^{\pi_0^t}\right)(s_1) \\
&=2\sum_{t=0}^T\underset{\Hc^{t-1}}{\Eb}\left[\underset{f^*}{\Eb}\left(V_{1;f^t}^{\pi_0^t}-V_{1;f^*}^{\pi_0^t}\right)(s_1)\middle|\Hc^{t-1}\right] \\
&=2\sum_{t=0}^T\underset{f^*}{\Eb}\left(V_{1;f^t}^{\pi_0^t}-V_{1;f^*}^{\pi_0^t}\right)(s_1).
\end{align*}

{\bf Step II: Bounding the regret by Bellman error.} By standard policy loss decomposition (Lemma~\ref{lemma:loss-decomp}), we have
\begin{align*}
\mathrm{Regret}(T)&\leq 2\sum_{t=0}^T\underset{f^*}{\Eb}\left(V_{1;f^t}^{\pi_0^t}-V_{1;f^*}^{\pi_0^t}\right)(s_1) =2\sum_{t=0}^T\sum_{h=1}^H\Ec(f^t,\pi_0^t,h).
\end{align*}

{\bf Step III: Bounding cumulative Bellman error using DE dimension.} This step is the novel part in the analysis of model-free RLHF with general function approximation, and we elaborate this in detail in Section~\ref{sec:key-idea}. 

At high level, we show the squared Bellman error of the function selected in round $t$ is bounded, i.e., $\sum_{i=1}^{t}\underset{(s_h,a_h)\sim\pi_i}{\Eb}[(f_h^i-\Tc_h f_{h+1}^i)(s_h,a_h)]^2\leq \mathcal{O}(\beta)$. Then by a pigeon-hold principle in Lemma~\ref{lemma:dist-eluder-sim}, the Bellman error $\sum_{i=1}^{t}\underset{(s_h,a_h)\sim\pi_i}{\Eb}[(f_h^i-\Tc_h f_{h+1}^i)(s_h,a_h)]^2$ can be bounded by the term $\mathcal{O}(\sqrt{td\beta})$ where $d$ is the BE dimension.

Finally, summing over horizon $h\in[H]$ completes the proof.


\if{0}
\section{Simulation results}
In this section, we evaluate the model-free Thompson algorithm for online RLHF described by Algorithm~\ref{alg:main}. We introduce the basic setup of the simulation and highlight the practical design below.

We consider a randomly generated MDP and the distribution $q(Q)$ of $Q$ values using a mixture of Gaussian (MoG) with $k$ mixture components.  Each component is initialized as a normal distribution $\Nc(\mathbf{0},\sigma^2\mathbf{I})$, where the standard deviation $\sigma=10$. Instead of using maximum likelihood estimator as in line 3 of Algorithm~\ref{alg:main}, we adopt the variational inference method, which maximizes the evidence lower bound (ELBO). Specifically, we maximize the ELBO objective $\Eb_{q}[\log p(\mathrm{feedback|Q})]-\mathrm{KL}(q(Q)||p(Q))$ in each round and update the parameters of MoG. Furthermore, we conduct mini-batch learning with batch size $n$ rather than online batch as employed in Algorithm~\ref{alg:main}. In other words, we collect $n$ trajectory comparison data ($2n$ trajectories in total) for each sampled $Q$ value. Then, during each iteration, we obtain $m$ samples of $Q$ values from the update MoG and use the sample estimate to approximate the ELBO objective value. Using gradient ascent for ELBO objective the parameters of the MoG is updated. 
This could potentially accelerate the learning process, improve the stability  and   convergence rate. In the experiment, $k=3$, $n=10$, and $m=20$.


Based on the practical design introduced above, we implement the Thompson sampling algorithm for RLHF and obtain the numerical results below. Fig.~\ref{fig:sim-V} shows that the value of the learned function fluctuates in the beginning and converges as the number of iterations increases. Moreover, the value of our learned function converges to 58, which is exactly the value of the underlying true model. Meanwhile, Fig.~\ref{fig:sim-avg-reg} demonstrates the average regret over iterations, where the average regret at iteration $t$ is defined as $\mathrm{Regret}(t)/t$. We observe that the average regret approaches 0 as the number of iterations increases. This asymptotic phenomenon illustrates the the effectiveness of Thomson sampling for online RLHF and coincides with our theoretical result.

\begin{figure}[h]
    \centering
    \begin{subfigure}{0.8\textwidth}
        \centering
        \includegraphics[width=\linewidth]{value.png}
        \caption{Value of the function over iterations.}
        \label{fig:sim-V}
    \end{subfigure}
    
    \vspace{1em} 

    \begin{subfigure}{0.8\textwidth}
        \centering
        \includegraphics[width=\linewidth]{avg-regret.png}
        \caption{Average regret over iterations.}
        \label{fig:sim-avg-reg}
    \end{subfigure}

    \caption{Comparison of value function and average regret over iterations.}
    \label{fig:sim-combined}
\end{figure}

\fi

\section{Conclusion}
In this work, we proposed a new model-free Thompson sampling algorithm for RLHF with general function approximation. We use Bellman eluder (BE) dimension as the complexity measure of the function class and characterize the complexity of our proposed algorithm. The regret bound scales in $\sqrt{T}$, and other multiplicative factor depends polynomially on the horizon $H$, the BE dimension $d$ and the $\log$-bracketing number of the function class. Towards the analysis, we first established the concentration-type inequality of the squared Bellman error based on the MLE generalization bound, which plays the crucial rules in obtaining the eluder-type regret bound and may be of independent interest.

\section*{Acknowledgment}
This material is based upon work supported by the Air Force Office of Scientific Research under award
number FA9550-21-1-0085, and under NSF grant \#2207759.

\appendix

\section{Related works}

{\bf RLHF algorithms.} Proximal Policy Optimization (PPO)~\cite{Schulman:PPO:2017} is the most popular algorithm in large language models (LLMs).However,  PPO suffers from instability, inefficiency, and high sensitivity to both hyperparameters~\cite{Choshen:PPO:2019} and code-level optimizations~\cite{Engstrom:PPO:2020}, which make it difficult to achieve the optimal performance in Chat-GPT4~\cite{ChatGPT4-report} and replicate its performance. Further, it requires the integration of additional components including a reward model, a value network (critic), and a reference model, potentially as large as the aligned LLM~\cite{Ouyang:ChatGPT:2022,Touvron:2023}. To resolve these aforementioned limitations, researchers have explored alternative strategies for LLM alignment. One approach is the reward-ranked finetuning (RAFT) that iteratively finetunes the model on the best outputs from a set of generated responses to maximize reward~\cite{dong2023raft,2023arXiv230405302Y,Touvron:2023,2023arXiv230808998G}. Another line of research builds upon the KL-regularized formulation~\cite{2023arXiv230518290R,2023arXiv230602231Z,2023arXiv230916240W,2023arXiv230906657L,2023arXiv231010505L}. Notably, Direct Preference Optimization (DPO)~\cite{2023arXiv230518290R} has been widely adopted as a promising alternative to PPO with notable stability and competitive performance. Similar to DPO, some other works also aims to optimize the LLM directly from the preference data~\cite{2023arXiv230510425Z,2023arXiv231012036G} but the theoretical analysis is still lacking.

{\bf RLHF theory.} Dueling bandits are simplified settings within the RLHF framework where the learner proposes two actions and only provide preference feedback. Numerous works have studied the dueling bandits~\cite{Yue:bandit:2012,Saha:bandit:2021,Bengs:Bandit:2021} and contextual dueling bandits~\cite{Dudik:bandit:2015,Saha:bandit:2022,Yue:bandit:2024}. RLHF has been first investigated in the tabular setting~\cite{Yichong:RLHF:2020,Ellen:RLHF:2019,Saha:RLHF:2023}. Beyond tabular setting, \cite{Xiaoyu:RLHF:2022} proposed the first confidence-set based algorithm for the online RLHF with general function approximation. Later \cite{Yuanhao:RLHF:2023} proposed a framework that transforms exiting sample-efficient RL algorithms to efficient online RLHF algorithms. Additionally, \cite{Wenhao:RLHF:2023,Zihao:RLHF:2023} investigated general function approximation for offline RLHF. Additionally, \cite{Xiong:RLHF:2023} introduced hybrid RLHF. The most relevant work~\cite{Runzhe:RLHF:2023} developed algorithms for online RLHF for both linear function approximation and nonlinear function approximation. Despite both their work and ours leverage the TS idea, they propose a model-based algorithm while our algorithm is model-free. Further, the nonlinear function approximation in their work and the general function approximation proposed in our work is fundamentally different.

{\bf Thompson sampling (TS).} Thompson sampling is a popular class of algorithms first proposed for bandit problem~\cite{Thompson:TS:1933}, and has strong empirical performance~\cite{Osband:TS:2018,IshFaq:TS:2023}. The class of Thompson sampling algorithms can be regarded as value based method, but the mechanism for exploration is fundamentally different from UCB-based algorithms. Thompson sampling has been observed to perform better than UCB algorithms, and the Bayesian community has developed posterior approximation techniques to help sample from the posterior. Despite its empirical success, the theoretical analysis is still limited. Thomson sampling has been investigated theoretical learning problems including bandit problems~\cite{Agrawal:TS:2012,Agrawal:TS:2017,Agrawal:TS:2013,SiweiWang:TS:2018} and RL theory~\cite{Russo:TS:2019,Zanette:TS:2019,IshFaq:TS:2021,ZhihanXiong:TS:2022,Wenhao:RLHF:2023,Zihao:RLHF:2023}. However, it is unclear if Thompson sampling can achieve optimal worst cast frequentist regret bound for the general case (even for the linear bandits)~\cite{Agrawal:TS:2013,Foster:bandit:ICML:2020}, even though some results for Bayesian regret are known where the regret averages over a known prior distribution~\cite{Russo:bandit:2014}.

\section{Squared Bellman error}\label{app:version-space}

In this section, we introduce the following two lemmas. The first lemma shows that with high probability any function in the confidence set has low Bellman error over the collected datasets $\Dc_1,\ldots,\Dc_H$ as well as the the distributions from which $\Dc_1,\ldots,\Dc_H$ are sampled.

\begin{lemma}[Informal]\label{lemma:concentration-1}
For any $(f_1,\ldots,f_h)\in\Vc_1^\Fc$, where $\Vc_1^\Fc$ is the confidence set defined in (\ref{eqn:app:version-1}), with probability at least $1-\delta$, for all $(t,h)\in[T]\times[H]$ we have
\begin{itemize}
\item[(1)] $\sum_{i=1}^{t-1}\underset{x\sim\pi_i^0}{\Eb}[(f_h^t(x)-(\Tc_h f_{h+1}^t)(x))^2]\leq \mathcal{O}\left(\beta_{\Fc_h}(t)+\beta_{\Fc_{h+1}}(t)\right)$.
\item[(2)] $\sum_{x\in\Dc_h^{t-1}}(f_h^t(x)-(\Tc_h f_{h+1}^t)(x))^2\leq \mathcal{O}\left(\beta_{\Fc_h}(t)+\beta_{\Fc_{h+1}}(t)\right)$.
\end{itemize}
Here $\pi_i^0$ stands for the greedy policy selected in round $i$, $\Dc_h^{t-1}=\{(s_h^j,a_h^j)\}_{j\in\{0,1\}}$ is the dataset at step $h$ up to round $t$, and $\beta_{\Fc_h}(t)=98\kappa^2\log(2HN_{[]}(\bar{\kappa}(2t)^{-1},\Fc_h,\norm{\cdot}_\infty)/\delta)$.
\end{lemma}

The second lemma guarantee that the underlying action value function $f^*$ lies in the confidence set with high probability. Consequently, the action value function $f^t$ selected in round $t$ shall be an upper bound of $Q^*$ with high probability.

\begin{lemma}[Informal]\label{lemma:concentration-3}
Under Assumption~\ref{assump:binary-para} holds, with probability at least $1-\delta$ it holds that $f^*\in\Vc_1^\Fc$ for all $t\in[T]$.
\end{lemma}

Although Lemma~\ref{lemma:concentration-1} and Lemma~\ref{lemma:concentration-3} are similar to Lemma~39 and Lemma~40 in~\cite{Chi_Jin:RL_eluder:2021}, the proofs are completely different. Thanks to the Thompson sampling, we first obtain MLE generalization bound for the MLE estimator. Next, we construct a confidence set based on such generalization bound and establish the concentration-type inequalities. 

The formal statement of Lemma~\ref{lemma:concentration-3} and Lemma~\ref{lemma:concentration-1} are provided in Lemma~\ref{app:version-space-2} and Lemma~\ref{lemma:concentration-2}.

\subsection{MLE generalization bound}\label{app:MLE}
The setup of Bayesian online conditional probability estimation problem was first proposed in~\cite{Runzhe:RLHF:2023}. For completeness, we briefly introduce the problem below.

Let $\mathcal{X}$ and $\mathcal{Y}$ be the instance space and target space. Let $\Fc:\mathcal{X}\times\mathcal{Y}\to \Rb$ be a function class. Consider $T$-round interaction. At the beginning of round $t\in[T]$, we observe an instance $x_t\in\mathcal{X}$ and aim to predict $f_t\in\Fc$. Then, the label $y_t\in\mathcal{Y}$ is revealed. Here $x_i\sim\Dc_i$ for some history-dependent data distribution $\Dc_i$, and $y_i\sim f^*(x,\cdot)$. Further, $f^*$ is sample from some known prior distribution $\rho\in\Delta(\Fc)$. 

Let $\mathcal{H}_t=\{x_1,f_1,y_1,x_2,f_2,y_2,\dots.x_t,f_t,y_t\}$ be the history up to round $t$. Define the MLE estimator over dataset $\{(x_s,y_s)\}_{s\in[t-1]}$ as
\begin{align*}
\widehat{f}_t=\arg\max_{f\in\Fc}\sum_{s=1}^{t-1}\log f(x_s,y_s).
\end{align*}

Assume the function class $\Fc$ is parameterized by a function class $\Gc$ via a link function $\Phi$. The rationale is to capture the structure of $\Fc$ in the learning problem in which the preference feedback is parameterized by a function approximation $f$ generated via a link function $\Phi$.

\begin{assumption}\label{assump:binary-para}
Assume $\mathcal{Y}=\{0,1\}$ is binary and function class $\Gc\subseteq \mathcal{X}\to [0,G]$ that parameterizes $\Fc$ via a link function $\Phi$. Specifically, we assume
\begin{align*}
\Fc=\{f(x,0)=\Phi(g(x)),f(x,1)=1-\Phi(g(x)):g\in\Gc\},
\end{align*}
where we further assume $\Phi$ satisfies Assumption~\ref{assump:link-fcn}. 
\end{assumption}

\begin{lemma}[MLE generalization bound (Lemma~C.5 in \cite{Runzhe:RLHF:2023})]\label{lemma:mle-gen}
If Assumption~\ref{assump:binary-para} holds, then with probability at least $1-\delta$, it holds that
\begin{align*}
\sum_{t=1}^T\underset{\Hc_{t-1}}{\Eb}\left[\underset{g,g'}{\Eb}\left[(g(x_t)-g'(x_t))^2\middle|\Hc_{t-1}\right]\right]\leq 4\gamma_\Gc(T)  
\end{align*}
where $\gamma_\Gc(t)=98\kappa^2\log(2GN_{[]}(\bar{\kappa}(t|\mathcal{Y}|)^{-1},\Gc,\norm{\cdot}_\infty)/\delta)$, and $g,g'$ are sampled from the posterior of $g^*$ conditioning on $\Hc_{t-1}$ in the inner expectation.
\end{lemma}

\subsection{Confidence set and squared Bellman error}\label{app:version-space-2}
Let the function class considered in Appendix~\ref{app:MLE} be the general function approximation classes $\Fc_1,\ldots,\Fc_H$. We define the confidence set as
\begin{align}
\Vc_1^\Fc=\left\{(f_1,\ldots,f_H)\in\Fc:\sum_{t=1}^T\underset{\Hc_{t-1}}{\Eb}\left[\underset{f_h}{\Eb}\left[(f_h(x_t)-\widehat{f}_h(x_t))^2\middle|\Hc_{t-1}\right]\right]\leq \beta_{\Fc_h}(T),\forall h\right\},  \label{eqn:app:version-1}
\end{align}
where $\beta_{\Fc_h}(t)=98\kappa^2\log(2HN_{[]}(\bar{\kappa}(2t)^{-1},\Fc_h,\norm{\cdot}_\infty)/\delta)$, $\widehat{f}_h$ is the maximum likelihood estimator at step $h$, $f_h$ is sampled from the posterior of $f_{h}^*$ conditioning on $\Hc_{t-1}$ in the inner expectation, and $x_t=(s_h,a_h)$ is the $t$-th round state-action pair at step $h$.

\begin{lemma}[\cite{Wenhao_Zhan:RLHF:2024}]\label{lemma:version-1}
Under Assumption~\ref{assump:binary-para}, define confidence set $\Vc_1^\Fc$ in (\ref{eqn:app:version-1}). Then the following statements holds:
\begin{itemize}
\item[(1)] $f^*\in\Vc^{\Fc}$ with probability at least $1-\delta$.
\item[(2)] For any $f,f'\in\Vc^{\Fc}$ where $f=(f_1,f_2,\ldots,f_H)$ and $f'=(f'_1,f'_2,\ldots,f'_H)$, we have $\sum_{t=1}^T\underset{\Hc_{t-1}}{\Eb}\left[\underset{f_h}{\Eb}\left[(f_h(x_t)-\widehat{f}_h(x_t))^2\middle|\Hc_{t-1}\right]\right]\leq 4\beta_{\Fc_h}(T)$. with probability at least $1-\delta$.
\end{itemize}
\end{lemma}
Here $\beta_{\Fc_h}(\cdot)$ is slightly different from $\gamma_{\Fc_h}(\cdot)$ in that there is an additional $H$ factor in the $\log$-term, which is due to the union bound over horizon $H$. Despite the we claim the two statements for the total round $T$, clearly they hold for each round $t\in[T]$.

We point it out that $f$ obtained by greedy selection in the algorithm plays the role of MLE maximizer $\widehat{f}$. Since $f^*\in\Vc^\Fc$ with high probability, we therefore can bound term 
\begin{align}
\sum_{t=1}^T\underset{\Hc_{t-1}}{\Eb}\left[(f_h^*(x_t)-\widehat{f}_h(x_t))^2\middle|\Hc_{t-1}\right]\leq 4\beta_{\Fc_h}(T). \label{eqn:mle-bound-1}
\end{align}

Recall that $\widehat{f}$ is exactly the MLE maximizer in the algorithm (we use $f$ instead of $\hat{f}$ in the algorithm). To this end, we are ready to show Lemma~\ref{lemma:concentration-2}, a stronger version of Lemma~\ref{lemma:concentration-1}.

\if{0}
The second version space is inspired by the Bellman eluder dimension~\cite{Chi_Jin:RL_eluder:2021}, which serves as the cornerstone in the regret analysis. Formally, we define the second version space $\Vc_2^\Fc$ as
\begin{align}
\Vc_2^\Fc=\left\{(f_1,\ldots,f_H)\in\Fc:\Lc_{\Dc_h}(f_h,\widehat{f}_{h+1})\leq \inf_{g\in\Fc_h }\Lc_{\Dc_h}(g,\widehat{f}_{h+1})+\alpha_h \right\},
\label{eqn:app:version-2}
\end{align}
where $\Dc_h$ is the dataset at step $h$ consisting both state-action pairs $(s_{t,h}^0,a_{t,h}^0)_t$ and $(s_{t,h}^1,a_{t,h}^1)_t$, $\Lc_{\Dc_h}(\xi_h,\zeta_{h+1})=\sum_{x_t\in\Dc_h}\left(\xi_h(x_t)-(\Tc_h\zeta_{h+1})(x_t)\right)^2$, and $\alpha_h=$.

To build up the connection between $\Vc_1^\Fc$ and $\Vc_2^\Fc$, we first construct an auxiliary version space $\widetilde{\Vc}_2^\Fc$ as follows
\begin{align}
\widetilde{\Vc}_2^\Fc=\left\{(f_1,\ldots,f_H)\in\Fc:\Lc_{\Dc_h}(f_h,f^*_{h+1})\leq \inf_{g\in\Fc_h }\Lc_{\Dc_h}(g,f^*_{h+1})+\widetilde{\alpha}_h \right\},
\label{eqn:app:version-3}
\end{align}
where $\widetilde{\alpha}_h=c\log(TH\Nc_{[]}(1/K,\Fc_h,\norm{\cdot}_\infty)/\delta)$ for some absolute constant $c$. In genie-aided version space $\widetilde{\Vc}_2^\Fc$, the underlying true action value function $f^*$ is adopted rather than the MLE maximizer $\widehat{f}_{h+1}$, and the auxiliary version space $\Vc_2^\Fc$ is generated independently given the dataset and has nothing to do with our proposed algorithm. 
\begin{lemma}
Define version space $\Vc_2^\Fc$ in (\ref{eqn:app:version-3}), then the following statement holds:
\begin{itemize}
\item[(1)] $f^*\in\Vc^\Fc$ with probability at least $1-\delta$.
\item[(2)] For $f^k\in\argmax_{f\in\Vc_h}f_1(s_1,\pi_f(s_1))$, it holds that $\sum_{x\in\Dc_h}(f_h^k(x)-(\Tc_hf_{h+1}^k)(x))^2?\leq \mathcal{O}(\alpha_h)$.
\end{itemize}
\end{lemma}
The proof is similar to that of Lemma~39 and Lemma~40 in~\cite{Chi_Jin:RL_eluder:2021} and omitted here for brevity. Note that we replace the covering number $\Nc(\epsilon,\Fc,\norm{\cdot})$ in the original work by the bracketing number based on the inequality $\Nc(\epsilon,\Fc,\norm{\cdot})\leq \Nc_{[]}(2\epsilon,\Fc,\norm{\cdot})$.
\fi

\begin{lemma}\label{lemma:concentration-2}
For any $(f_1,\ldots,f_h)\in\Vc_1^\Fc$, where $\Vc_1^\Fc$ is defined in (\ref{eqn:app:version-1}), with probability at least $1-\delta$, for all $f\in\Vc_1^\Fc$ and $h\in[H]$, we have
\begin{itemize}
\item[(1)] $\sum_{i=1}^{t-1}\underset{x\sim\pi_i^0}{\Eb}[(f_h(x)-(\Tc_h f_{h+1}^t)(x))^2]\leq \mathcal{O}\left(\beta_{\Fc_h}(t)+\beta_{\Fc_{h+1}}(t)\right)$.
\item[(2)] $\sum_{x\in\Dc_h^{t-1}}(f_h(x)-(\Tc_h f_{h+1})(x))^2\leq \mathcal{O}\left(\beta_{\Fc_h}(t)+\beta_{\Fc_{h+1}}(t)\right)$.
\end{itemize}
Here $\pi_i^0$ stands for the greedy policy selected in round $i$, and $\Dc_h^{t-1}=\{(s_h^j,a_h^j)\}_{j\in\{0,1\}}$ is the dataset at step $h$ up to round $t$.
\end{lemma}
\begin{proof}
We prove the second inequality first. Note that
\begin{align*}
&\sum_{i=1}^{t-1}\underset{x\sim\pi_i^0}{\Eb}\left[(f_h(x)-(\Tc f_{h+1})(x))^2\right] \\
&=\sum_{i=1}^{t-1}\underset{x\sim\pi_i^0}{\Eb}\left[(f_h(x)-(\Tc_hf_{h+1}^*)(x))^2\right]+\sum_{i=1}^{t-1}\underset{x\sim\pi_i^0}{\Eb}\left[((\Tc_h f_{h+1}^*)(x)-(\Tc_h f_{h+1})(x))^2\right] \\
&\qquad +\sum_{i=1}^{t-1}\underset{x\sim\pi_i^0}{\Eb}\left[(f_h(x)-(\Tc_h f_{h+1}^*)(x))((\Tc_h f_{h+1})(x)-(\Tc_h f_{h+1})(x))\right] \\
&=\frac{3}{2}\sum_{i=1}^{t-1}\underset{x\sim\pi_i^0}{\Eb}\left[(f_h(x)-(\Tc_hf_{h+1}^*)(x))^2\right]+\frac{3}{2}\sum_{i=1}^{t-1}\underset{x\sim\pi_i^0}{\Eb}\left[((\Tc_h f_{h+1}^*)(x)-(\Tc_h f_{h+1})(x))^2\right] \\
&\leq \frac{3}{2}\sum_{i=1}^{t-1}\underset{x\sim\pi_i^0}{\Eb}\left[(f_h(x)-f_{h}^*(x))^2\right]+\frac{3}{2}\sum_{i=1}^{t-1}\underset{x\sim\pi_i^0}{\Eb}\left[((\Tc_h f_{h+1}^*)(x)-(\Tc_h f_{h+1})(x))^2\right] \\
&\leq \frac{3}{2}\sum_{i=1}^{t-1}\underset{x\sim\pi_i^0}{\Eb}\left[(f_h(x)-f_{h}^*(x))^2\right]+\frac{3}{2}\sum_{i=1}^{t-1}\underset{x\sim\pi_i^0,y\sim P_h(\cdot|x)}{\Eb}\left[(f_{h+1}^*(y)-f_{h+1}(y))^2\right] \\
&\leq 6(\beta_{\Fc_h}(t)+\beta_{\Fc_{h+1}}(t))
,
\end{align*}
where the second equality follows from $\Eb[ab]\leq \frac{1}{2}(\Eb[a^2]+\Eb[b^2])$, and last inequality follows from (\ref{eqn:mle-bound-1})) and the fact that $\pi_i^0$ is the policy adopted for sampling. 

The second inequality can be proved based on (1) by Lemma~\ref{lemma:supp-1}.
\end{proof}

Despite the inequalities resemble the concentration inequalities provided in~\cite{Chi_Jin:RL_eluder:2021} (Lemma~39), they are derived directly from the constructed confidence set. Further, the guarantee established in~\cite{Chi_Jin:RL_eluder:2021} only holds for certain function approximation, while our result holds for all function approximation within the confidence set.

\section{Proof of Theorem~\ref{thm:main}}\label{app:thm_proof}
In this section, we assume $\Fc$ is the function class after completion.

Our proof proceeds in three steps. We start with simplifying the target regret term. Then, we show bound the simplified term by Bellman error. Finally, We use distributional eluder dimension to bound the cumulative Bellman error.

{\bf Step I: Simplifying the regret term.} Recall that $\pi_1^t=\pi_0^{t-1}$ for all $t\in[T]$, therefore it holds that
\begin{align*}
\mathrm{Regret}(T)&=\sum_{t=1}^T\left(2V_1^*-V_1^{\pi_0^t}-V_1^{\pi_1^t}\right)(s_1) \\
&=\sum_{t=1}^T\left(2V_1^*-V_1^{\pi_0^t}-V_1^{\pi_0^{t-1}}\right)(s_1) \\
&\leq 2\sum_{t=0}^T\left(V_1^*-V_1^{\pi_0^t}\right)(s_1).
\end{align*}
Here the value function is defined for the true underlying MDP. For ease of exposition, we define $V_{1;f}^\pi$ as the value function under policy $\pi$ with respect to a model $f\in\Fc$. Then, the regret can be expressed as
\begin{align*}
\mathrm{Regret}(T)&\leq 2\sum_{t=0}^T\underset{f^*}{\Eb}\left(V_{1;f^*}^*-V_{1;f^*}^{\pi_0^t}\right)(s_1).
\end{align*}

Denote $\mathcal{H}^{t}:=\{(\tau_0^s,\tau_1^s,o^s)\}_{s=1}^t$ as the history up to round $t$. Recall the posterior $p(f|\Dc^t)$ is the product of the prior and the likelihood, thus $f^t$ and $f^*$ are identically distributed given history $\Hc^{t-1}$. Based on the above observation and the fact that $\pi_0^t$ and $\pi^*$ are optimal policies of $f^t$ and $f^*$, we conclude that $V_{1;f^t}^{\pi_0^t}$ and $V_{1;f^*}^{\pi^*}$ are identically distributed given history $\Hc^{t-1}$, and we have $\Eb_{\Hc^{t-1}}\Eb_{f^*}[V_{1;f^t}^{\pi_0^t}|\Hc^{t-1}]=\Eb_{\Hc^{t-1}}\Eb_{f^*}[V_{1;f^*}^{\pi^*}|\Hc^{t-1}]$. Therefore, the regret can be further expressed as
\begin{align*}
\mathrm{Regret}(T)&\leq 2\sum_{t=0}^T\underset{f^*}{\Eb}\left(V_{1;f^*}^*-V_{1;f^*}^{\pi_0^t}\right)(s_1) \\
&=2\sum_{t=0}^T\underset{\Hc^{t-1}}{\Eb}\left[\underset{f^*}{\Eb}\left(V_{1;f^*}^*-V_{1;f^*}^{\pi_0^t}\right)(s_1)\middle|\Hc^{t-1}\right] \\
&=2\sum_{t=0}^T\underset{\Hc^{t-1}}{\Eb}\left[\underset{f^*}{\Eb}\left(V_{1;f^t}^{\pi_0^t}-V_{1;f^*}^{\pi_0^t}\right)(s_1)\middle|\Hc^{t-1}\right] \\
&=2\sum_{t=0}^T\underset{f^*}{\Eb}\left(V_{1;f^t}^{\pi_0^t}-V_{1;f^*}^{\pi_0^t}\right)(s_1).
\end{align*}

{\bf Step II: Bounding the regret by Bellman error.} By standard policy loss decomposition (Lemma~\ref{lemma:loss-decomp}), we have
\begin{align*}
\mathrm{Regret}(T)&\leq 2\sum_{t=0}^T\underset{f^*}{\Eb}\left(V_{1;f^t}^{\pi_0^t}-V_{1;f^*}^{\pi_0^t}\right)(s_1)  \\
&=2\sum_{t=0}^T\sum_{h=1}^H\Ec(f^t,\pi_0^t,h).
\end{align*}

{\bf Step III: Bounding cumulative Bellman error using DE dimension.} We then focus on a fixed step $h$ and bound the cumulative Bellman error $\sum_{t=1}^T\Ec(f^t,\pi_0^t,h)$ using Lemma~\ref{lemma:concentration-3} and Lemma~\ref{lemma:dist-eluder}.

Invoking Lemma~\ref{lemma:concentration-3}(1) and Lemma~\ref{lemma:dist-eluder} with 
\begin{align*}
\begin{cases}
&\rho=\frac{1}{T},\omega=\sqrt{\frac{1}{T}},C=1 \\
&\mathcal{X}=\mathcal{S}\times\mathcal{A}, \Phi=\Fc_h-\Tc_h\Fc_{h+1}, \Pi=\Pi_{\Fc,h}, \\
&\phi_t=f_h^t-\Tc_hf_{h+1}^t, \mbox{ and } \mu_t=\Pb^{\pi_0^t}(s_h=\cdot,a_h=\cdot),
\end{cases}
\end{align*}
gives
\begin{align}
\sum_{i=1}^{t}\Ec(f^t,\pi_0^t,h)\leq \mathcal{O}\left(\sqrt{t\cdot\dim_{\mathrm{DE}}(\Fc_h-\Tc_h\Fc_{h+1},\Pi_{\Fc,h},\sqrt{1/K})\cdot\iota_h^t}\right), \label{eqn:result-1}
\end{align}
where $\iota_h^t=\beta_{\Fc_h}(t)+\beta_{\Fc_{h+1}}(t)$ and $\beta_{\Fc_h}(t)=98\kappa^2\log(2HN_{[]}(\bar{\kappa}(2t)^{-1},\Fc_h,\norm{\cdot}_\infty)/\delta)$.

Let $\beta_\Fc=\sup_{h\in[H]}\beta_{\Fc_h}(T)$ and recall the definition of Bellman eluder dimension $\dim_{\mathrm{BE}}(\Fc,\Pi_\Fc,\sqrt{1/K})=\sup_{h\in[H]}\dim_{\mathrm{DE}}(\Fc_h-\Tc_h\Fc_{h+1},\Pi_{\Fc,h},\sqrt{1/K})$, we have
\begin{align*}
\sum_{i=1}^{t}\Ec(f^t,\pi_0^t,h)\leq \mathcal{O}\left(\sqrt{t\cdot\dim_{\mathrm{BE}}(\Fc,\Pi_\Fc,\sqrt{1/K})\cdot\beta_\Fc)}\right)
\end{align*}

Alternatively, invoking Lemma~\ref{lemma:concentration-3}(1) and Lemma~\ref{lemma:dist-eluder} with 
\begin{align*}
\begin{cases}
&\rho=\frac{1}{T},\omega=\sqrt{\frac{1}{T}},C=1 \\
&\mathcal{X}=\mathcal{S}\times\mathcal{A}, \Phi=\Fc_h-\Tc_h\Fc_{h+1}, \Pi=\Pi_{\Delta,h}, \\
&\phi_t=f_h^t-\Tc_hf_{h+1}^t, \mbox{ and } \mu_t=\mathbf{1}\{\cdot=(s_h^i,a_h^i)\},
\end{cases}
\end{align*}
(with abuse of notation $(s_h^i,a_h^i)$ consists of $2i$ trajectories up to round $i$) gives
\begin{align}
\sum_{i=1}^{t}\Ec(f^t,\pi_0^t,h)\leq \mathcal{O}\left(\sqrt{t\cdot\dim_{\mathrm{DE}}(\Fc_h-\Tc_h\Fc_{h+1},\Pi_{\Fc,h},\sqrt{1/K})\cdot\iota_h^t}\right), \label{eqn:result-2}
\end{align}
where $\iota_h^t=\beta_{\Fc_h}(t)+\beta_{\Fc_{h+1}}(t)$ and $\beta_{\Fc_h}(t)=98\kappa^2\log(2HN_{[]}(\bar{\kappa}(2t)^{-1},\Fc_h,\norm{\cdot}_\infty)/\delta)$.

Let $\beta_\Fc=\sup_{h\in[H]}\beta_{\Fc_h}(T)$ and recall the definition of Bellman eluder dimension $\dim_{\mathrm{BE}}(\Fc,\Pi_\Fc,\sqrt{1/K})=\sup_{h\in[H]}\dim_{\mathrm{DE}}(\Fc_h-\Tc_h\Fc_{h+1},\Pi_{\Fc,h},\sqrt{1/K})$, we have
\begin{align*}
\sum_{i=1}^{t}\Ec(f^t,\pi_0^t,h)\leq \mathcal{O}\left(\sqrt{t\cdot\dim_{\mathrm{BE}}(\Fc,\Pi_\Fc,\sqrt{1/K})\cdot\beta_\Fc)}\right)
\end{align*}

Summing over horizon $h\in[H]$ completes the proof.

\section{Technical lemmas}
\begin{lemma}[\cite{YinglunZhu:RL:NIPS:2022}]\label{lemma:supp-1}
Let $\{X_t\}_{t\leq T}$ be a sequence of positive valued random variables adapted to filtration $\Fc_t$, and let $\Eb_t[\cdot]=\Eb[\cdot|\Fc_{t-1}]$. If $X_t\leq B$ almost surely, then with probability at least $1-\delta$, the following holds:
\begin{align*}
&\sum_{t=1}^T X_t\leq \frac{3}{2}\sum_{t=1}^T\Eb_t[X_t]+4B\log(1/\delta), \\
&\sum_{t=1}^T\Eb_t[X_t]\leq 2\sum_{t=1}^T X_t+8B\log(1/\delta).
\end{align*}
\end{lemma}

\begin{lemma}[Policy loss decomposition (Lemma 1 in~\cite{NanJiang:bellman-rank:2017})]\label{lemma:loss-decomp}
Let $f\in\Fc$ be a function and $\pi_f$ be the associated greedy policy. $V_{1;f}(s_1)=\Eb[f(s_1,\pi_f(s_1))]$, then it holds that
\begin{align*}
(V_{1;f}^{\pi_f}-V_{1;f^*}^{\pi_f})(s_1)=\sum_{h=1}^H\Ec(f,\pi_f,h),
\end{align*}
where $\Ec(f,\pi_f,h)={\Eb}_{(s_h,a_h)\sim\pi_f}[(f_h-\Tc_hf_{h+1})(s_h,a_h)]$.
\end{lemma}

\begin{lemma}[Lemma 41 in~\cite{Chi_Jin:RL_eluder:2021}]\label{lemma:dist-eluder}
Given a function class $\Phi$ defined on $\mathcal{X}$ with $|\phi(x)|\leq C$ for all $(g,x)\in\Phi\times\mathcal{X}$, and a family of probability measure $\Pi$ over $\mathcal{X}$. Suppose sequence $\{\phi_k\}_{k=1}^K\subseteq \Phi$ and $\{\mu_k\}_{k=1}^K\subseteq \Pi$ satisfy that for all $k\in[K]$, $\sum_{t=1}^{k-1}(\Eb_{\mu_t}[\phi_k])^2\leq \beta$. Then for all $k\in[K]$ and $\omega >0$,
\begin{align*}
\sum_{t=1}^k|\Eb_{\mu_t}[\phi_t]|\leq \mathcal{O}\left(\sqrt{\dim_{\mathrm{DE}}(\Phi,\Pi,\omega)\beta k}+\min\{k,\dim_{\mathrm{DE}}(\Phi,\Pi,\omega)\}C+k\omega\right).
\end{align*}
\end{lemma}

\section{Simulation results}
In this section, we evaluate the model-free Thompson algorithm for online RLHF described by Algorithm~\ref{alg:main}. We introduce the basic setup of the simulation and highlight the practical design below.

We consider a randomly generated MDP of size $k=15$ ($|\mathcal{S}|\times|\mathcal{A}|$) and the distribution of $Q(s,a)$ for any $(s,a)$  follows a normal distribution. Instead of using maximum likelihood estimator as in line 3 of Algorithm~\ref{alg:main}, we adopt the popular variational inference method, which maximizes the evidence lower bound (ELBO), to update the posterior and policy. Specifically, we maximize the ELBO objective $\Eb_{q}[\log p(\mathrm{feedback|Q})]-\mathrm{KL}(q(Q)||p(Q))$ in each round and update the distribution's parameters~\cite{2019arXiv190602691K}. We conduct mini-batch learning with batch size $n$ rather than online batch (with batch size 1) as employed in Algorithm~\ref{alg:main} to accelerate and stablize the learning process.. In other words, we collect $n$ trajectory comparison data ($2n$ trajectories in total) for each sampled $Q$ value. To approximate the ELBO objective, we obtain $m$ new samples of $Q$ values during each iteration. Here the samples generated in the history are helpful and ideally ELBO objective is estimated based on all samples collected so far. Considering the computational efficiency, instead we use the latest $\ell$ samples estimate to approximate the ELBO objective value, and use all samples in the first few iterations since there are insufficient samples. Finally, we maximize the approximate ELBO objective based on which we update the $Q$ function and corresponding optimal policy. To reduce noise and fluctuations in data, we apply a smoothing window of size $\omega$ in the figures shown below, which essentially calculate the average within a window of certain size.  In the experiment, $k=15$, $n=5$, $m=20$, $\ell=50$ and $\omega=20$.

Based on the practical design introduced above, we implement the Thompson sampling algorithm for RLHF and obtain the numerical results in~Fig.~\ref{fig:sim}. In the first figure, we evaluate per-step regret, which evaluates the performance gap between the current policy and the optimal policy. We see that the per-step regret converges to zero, which indicates the policy gradually becomes optimal. The second figure demonstrates the average regret over iterations, where the average regret at iteration $t$ is defined as $\mathrm{Regret}(t)/t$. We observe that the average regret approaches 0 as the number of iterations increases but the rate is slower than the per-step regret. Further, the third figure shows that the value of the learned function converges as the number of iterations increases. In fact, the value of our learned function converges to approximately 65, which is the value under the underlying MDP randomly generated in the beginning. This asymptotic phenomenon illustrates the the effectiveness of Thomson sampling for online RLHF and coincides with our theoretical result. Further, we notice that the converged value of the $Q$ parameters (mean and standard deviation of the normal distribution) is different from the underlying model, but they generate similar optimal policy and achieve the same optimal value.

\if{0}
\begin{figure}[b]
\centering
\includegraphics[width=\linewidth]{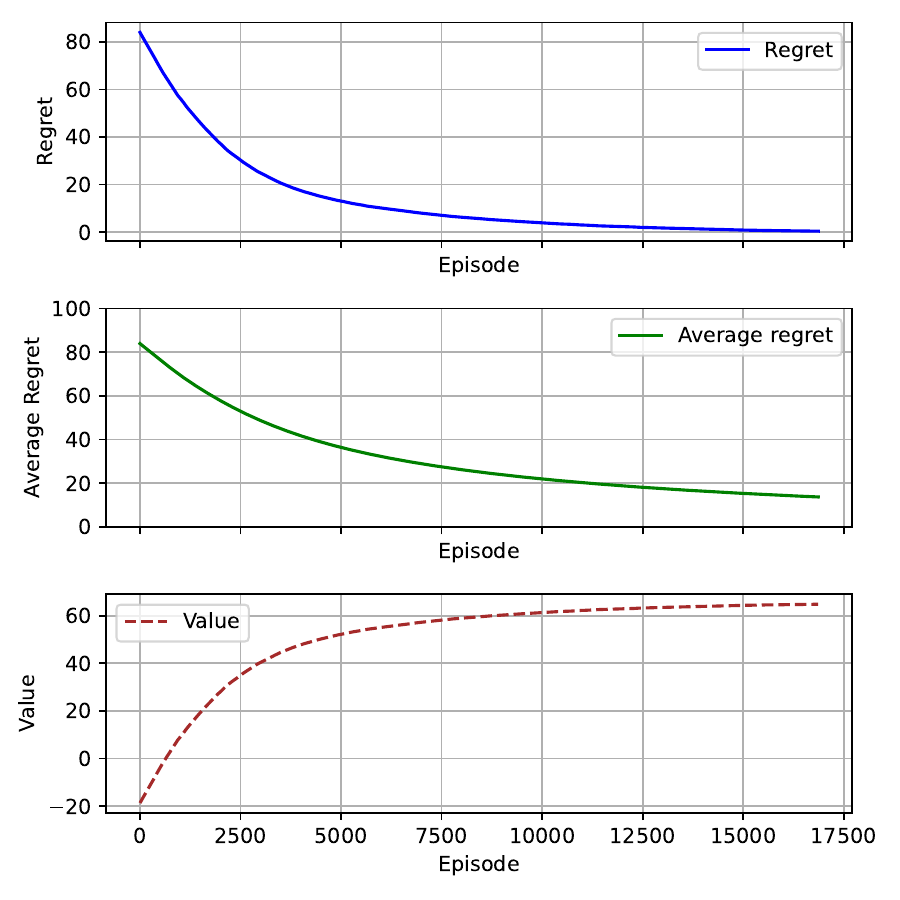}
\caption{Average regret and value of the function over iterations.}
\label{fig:sim}
\end{figure}
\fi


\begin{figure}[htpb]
\centering
\includegraphics[width=\linewidth]{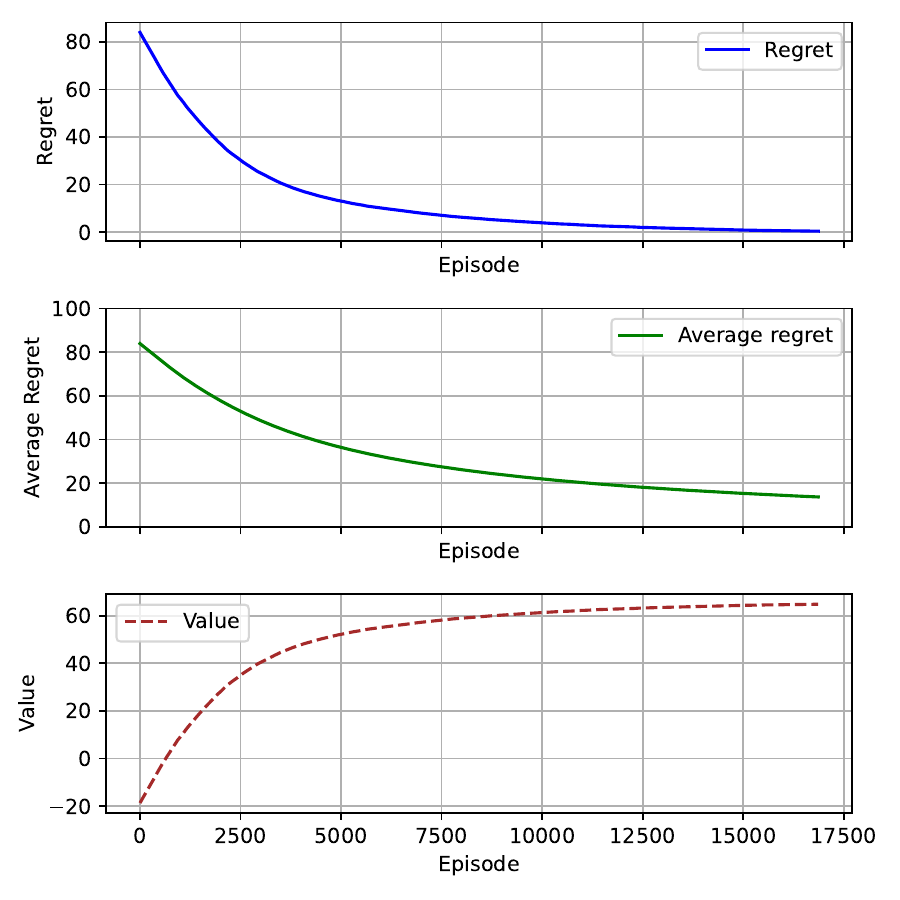}
\caption{Average regret and value of the function over iterations.}
\label{fig:sim}
\end{figure}

\newpage
\bibliography{ref}

\begin{thebibliography}{10}

\bibitem{Agrawal:TS:2012}
Shipra Agrawal and Navin Goyal.
\newblock Analysis of thompson sampling for the multi-armed bandit problem.
\newblock In Shie Mannor, Nathan Srebro, and Robert~C. Williamson, editors, {\em Proceedings of the 25th Annual Conference on Learning Theory}, volume~23 of {\em Proceedings of Machine Learning Research}, pages 39.1--39.26, Edinburgh, Scotland, 25--27 Jun 2012. PMLR.

\bibitem{Agrawal:TS:2013}
Shipra Agrawal and Navin Goyal.
\newblock Thompson sampling for contextual bandits with linear payoffs.
\newblock In Sanjoy Dasgupta and David McAllester, editors, {\em Proceedings of the 30th International Conference on Machine Learning}, volume~28 of {\em Proceedings of Machine Learning Research}, pages 127--135, Atlanta, Georgia, USA, 17--19 Jun 2013. PMLR.

\bibitem{Agrawal:TS:2017}
Shipra Agrawal and Navin Goyal.
\newblock Near-optimal regret bounds for thompson sampling.
\newblock {\em J. ACM}, 64(5), September 2017.

\bibitem{Bengs:Bandit:2021}
Viktor Bengs, R\'{o}bert Busa-Fekete, Adil El~Mesaoudi-Paul, and Eyke H\"{u}llermeier.
\newblock Preference-based online learning with dueling bandits: a survey.
\newblock {\em J. Mach. Learn. Res.}, 22(1), January 2021.

\bibitem{Bradley:BTL:1952}
Ralph~Allan Bradley and Milton~E. Terry.
\newblock Rank analysis of incomplete block designs: The method of paired comparisons.
\newblock {\em Biometrika}, 39(3-4):324--345, 12 1952.

\bibitem{Casper:RLHF:2023}
Stephen {Casper}, Xander {Davies}, Claudia {Shi}, Thomas {Krendl Gilbert}, J{\'e}r{\'e}my {Scheurer}, Javier {Rando}, Rachel {Freedman}, Tomasz {Korbak}, David {Lindner}, Pedro {Freire}, Tony {Wang}, Samuel {Marks}, Charbel-Rapha{\"e}l {Segerie}, Micah {Carroll}, Andi {Peng}, Phillip {Christoffersen}, Mehul {Damani}, Stewart {Slocum}, Usman {Anwar}, Anand {Siththaranjan}, Max {Nadeau}, Eric~J. {Michaud}, Jacob {Pfau}, Dmitrii {Krasheninnikov}, Xin {Chen}, Lauro {Langosco}, Peter {Hase}, Erdem {B{\i}y{\i}k}, Anca {Dragan}, David {Krueger}, Dorsa {Sadigh}, and Dylan {Hadfield-Menell}.
\newblock {Open Problems and Fundamental Limitations of Reinforcement Learning from Human Feedback}.
\newblock {\em arXiv e-prints}, page arXiv:2307.15217, July 2023.

\bibitem{Xiaoyu:RLHF:2022}
Xiaoyu {Chen}, Han {Zhong}, Zhuoran {Yang}, Zhaoran {Wang}, and Liwei {Wang}.
\newblock {Human-in-the-loop: Provably Efficient Preference-based Reinforcement Learning with General Function Approximation}.
\newblock {\em arXiv e-prints}, page arXiv:2205.11140, May 2022.

\bibitem{Choshen:PPO:2019}
Leshem {Choshen}, Lior {Fox}, Zohar {Aizenbud}, and Omri {Abend}.
\newblock {On the Weaknesses of Reinforcement Learning for Neural Machine Translation}.
\newblock {\em arXiv e-prints}, page arXiv:1907.01752, July 2019.

\bibitem{Christinano:RLHF:2017}
Paul~F. Christiano, Jan Leike, Tom~B. Brown, Miljan Martic, Shane Legg, and Dario Amodei.
\newblock Deep reinforcement learning from human preferences.
\newblock In {\em Proceedings of the 31st International Conference on Neural Information Processing Systems}, NIPS'17, page 4302–4310, Red Hook, NY, USA, 2017. Curran Associates Inc.

\bibitem{dong2023raft}
Hanze Dong, Wei Xiong, Deepanshu Goyal, Yihan Zhang, Winnie Chow, Rui Pan, Shizhe Diao, Jipeng Zhang, KaShun SHUM, and Tong Zhang.
\newblock {RAFT}: Reward ranked finetuning for generative foundation model alignment.
\newblock {\em Transactions on Machine Learning Research}, 2023.

\bibitem{Dudik:bandit:2015}
Miroslav Dudík, Katja Hofmann, Robert~E. Schapire, Aleksandrs Slivkins, and Masrour Zoghi.
\newblock Contextual dueling bandits.
\newblock In Peter Grünwald, Elad Hazan, and Satyen Kale, editors, {\em Proceedings of The 28th Conference on Learning Theory}, volume~40 of {\em Proceedings of Machine Learning Research}, pages 563--587, Paris, France, 03--06 Jul 2015. PMLR.

\bibitem{Engstrom:PPO:2020}
Logan {Engstrom}, Andrew {Ilyas}, Shibani {Santurkar}, Dimitris {Tsipras}, Firdaus {Janoos}, Larry {Rudolph}, and Aleksander {Madry}.
\newblock {Implementation Matters in Deep Policy Gradients: A Case Study on PPO and TRPO}.
\newblock {\em arXiv e-prints}, page arXiv:2005.12729, May 2020.

\bibitem{Foster:bandit:ICML:2020}
Dylan~J. Foster and Alexander Rakhlin.
\newblock Beyond ucb: optimal and efficient contextual bandits with regression oracles.
\newblock In {\em Proceedings of the 37th International Conference on Machine Learning}, ICML'20. JMLR.org, 2020.

\bibitem{Gao:RLHF:2023}
Leo Gao, John Schulman, and Jacob Hilton.
\newblock Scaling laws for reward model overoptimization.
\newblock In {\em Proceedings of the 40th International Conference on Machine Learning}, ICML'23. JMLR.org, 2023.

\bibitem{2023arXiv231012036G}
Mohammad {Gheshlaghi Azar}, Mark {Rowland}, Bilal {Piot}, Daniel {Guo}, Daniele {Calandriello}, Michal {Valko}, and R{\'e}mi {Munos}.
\newblock {A General Theoretical Paradigm to Understand Learning from Human Preferences}.
\newblock {\em arXiv e-prints}, page arXiv:2310.12036, October 2023.

\bibitem{2023arXiv230808998G}
Caglar {Gulcehre}, Tom {Le Paine}, Srivatsan {Srinivasan}, Ksenia {Konyushkova}, Lotte {Weerts}, Abhishek {Sharma}, Aditya {Siddhant}, Alex {Ahern}, Miaosen {Wang}, Chenjie {Gu}, Wolfgang {Macherey}, Arnaud {Doucet}, Orhan {Firat}, and Nando {de Freitas}.
\newblock {Reinforced Self-Training (ReST) for Language Modeling}.
\newblock {\em arXiv e-prints}, page arXiv:2308.08998, August 2023.

\bibitem{IshFaq:TS:2021}
Haque {Ishfaq}, Qiwen {Cui}, Viet {Nguyen}, Alex {Ayoub}, Zhuoran {Yang}, Zhaoran {Wang}, Doina {Precup}, and Lin~F. {Yang}.
\newblock {Randomized Exploration for Reinforcement Learning with General Value Function Approximation}.
\newblock {\em arXiv e-prints}, page arXiv:2106.07841, June 2021.

\bibitem{IshFaq:TS:2023}
Haque {Ishfaq}, Qingfeng {Lan}, Pan {Xu}, A.~{Rupam Mahmood}, Doina {Precup}, Anima {Anandkumar}, and Kamyar {Azizzadenesheli}.
\newblock {Provable and Practical: Efficient Exploration in Reinforcement Learning via Langevin Monte Carlo}.
\newblock {\em arXiv e-prints}, page arXiv:2305.18246, May 2023.

\bibitem{Jain:RLHF:2013}
Ashesh Jain, Brian Wojcik, Thorsten Joachims, and Ashutosh Saxena.
\newblock Learning trajectory preferences for manipulators via iterative improvement.
\newblock In C.J. Burges, L.~Bottou, M.~Welling, Z.~Ghahramani, and K.Q. Weinberger, editors, {\em Advances in Neural Information Processing Systems}, volume~26. Curran Associates, Inc., 2013.

\bibitem{NanJiang:bellman-rank:2017}
Nan Jiang, Akshay Krishnamurthy, Alekh Agarwal, John Langford, and Robert~E. Schapire.
\newblock Contextual decision processes with low {B}ellman rank are {PAC}-learnable.
\newblock In Doina Precup and Yee~Whye Teh, editors, {\em Proceedings of the 34th International Conference on Machine Learning}, volume~70 of {\em Proceedings of Machine Learning Research}, pages 1704--1713. PMLR, 06--11 Aug 2017.

\bibitem{Chi_Jin:RL_eluder:2021}
Chi Jin, Qinghua Liu, and Sobhan Miryoosefi.
\newblock Bellman eluder dimension: new rich classes of rl problems, and sample-efficient algorithms.
\newblock In {\em Proceedings of the 35th International Conference on Neural Information Processing Systems}, NIPS '21, Red Hook, NY, USA, 2021. Curran Associates Inc.

\bibitem{2019arXiv190602691K}
Diederik~P. {Kingma} and Max {Welling}.
\newblock {An Introduction to Variational Autoencoders}.
\newblock {\em arXiv e-prints}, page arXiv:1906.02691, June 2019.

\bibitem{Zihao:RLHF:2023}
Zihao {Li}, Zhuoran {Yang}, and Mengdi {Wang}.
\newblock {Reinforcement Learning with Human Feedback: Learning Dynamic Choices via Pessimism}.
\newblock {\em arXiv e-prints}, page arXiv:2305.18438, May 2023.

\bibitem{2023arXiv231010505L}
Ziniu {Li}, Tian {Xu}, Yushun {Zhang}, Zhihang {Lin}, Yang {Yu}, Ruoyu {Sun}, and Zhi-Quan {Luo}.
\newblock {ReMax: A Simple, Effective, and Efficient Reinforcement Learning Method for Aligning Large Language Models}.
\newblock {\em arXiv e-prints}, page arXiv:2310.10505, October 2023.

\bibitem{2023arXiv230906657L}
Tianqi {Liu}, Yao {Zhao}, Rishabh {Joshi}, Misha {Khalman}, Mohammad {Saleh}, Peter~J. {Liu}, and Jialu {Liu}.
\newblock {Statistical Rejection Sampling Improves Preference Optimization}.
\newblock {\em arXiv e-prints}, page arXiv:2309.06657, September 2023.

\bibitem{MacGlashan:RLHF:2017}
James MacGlashan, Mark~K Ho, Robert Loftin, Bei Peng, Guan Wang, David~L. Roberts, Matthew~E. Taylor, and Michael~L. Littman.
\newblock Interactive learning from policy-dependent human feedback.
\newblock In {\em Proceedings of the 34th International Conference on Machine Learning - Volume 70}, ICML'17, page 2285–2294. JMLR.org, 2017.

\bibitem{Ellen:RLHF:2019}
Ellen~R. Novoseller, Yanan Sui, Yisong Yue, and Joel~W. Burdick.
\newblock Dueling posterior sampling for preference-based reinforcement learning.
\newblock {\em ArXiv}, abs/1908.01289, 2019.

\bibitem{ChatGPT4-report}
{OpenAI}, Josh {Achiam}, Steven {Adler}, Sandhini {Agarwal}, Lama {Ahmad}, Ilge {Akkaya}, Florencia {Leoni Aleman}, Diogo {Almeida}, Janko {Altenschmidt}, Sam {Altman}, Shyamal {Anadkat}, Red {Avila}, Igor {Babuschkin}, Suchir {Balaji}, Valerie {Balcom}, Paul {Baltescu}, Haiming {Bao}, Mohammad {Bavarian}, Jeff {Belgum}, Irwan {Bello}, Jake {Berdine}, Gabriel {Bernadett-Shapiro}, Christopher {Berner}, Lenny {Bogdonoff}, Oleg {Boiko}, Madelaine {Boyd}, Anna-Luisa {Brakman}, Greg {Brockman}, Tim {Brooks}, Miles {Brundage}, Kevin {Button}, Trevor {Cai}, Rosie {Campbell}, Andrew {Cann}, Brittany {Carey}, Chelsea {Carlson}, Rory {Carmichael}, Brooke {Chan}, Che {Chang}, Fotis {Chantzis}, Derek {Chen}, Sully {Chen}, Ruby {Chen}, Jason {Chen}, Mark {Chen}, Ben {Chess}, Chester {Cho}, Casey {Chu}, Hyung~Won {Chung}, Dave {Cummings}, Jeremiah {Currier}, Yunxing {Dai}, Cory {Decareaux}, Thomas {Degry}, Noah {Deutsch}, Damien {Deville}, Arka {Dhar}, David {Dohan}, Steve {Dowling}, Sheila {Dunning}, Adrien {Ecoffet},
  Atty {Eleti}, Tyna {Eloundou}, David {Farhi}, Liam {Fedus}, Niko {Felix}, Sim{\'o}n {Posada Fishman}, Juston {Forte}, Isabella {Fulford}, Leo {Gao}, Elie {Georges}, Christian {Gibson}, Vik {Goel}, Tarun {Gogineni}, Gabriel {Goh}, Rapha {Gontijo-Lopes}, Jonathan {Gordon}, Morgan {Grafstein}, Scott {Gray}, Ryan {Greene}, Joshua {Gross}, Shixiang~Shane {Gu}, Yufei {Guo}, Chris {Hallacy}, Jesse {Han}, Jeff {Harris}, Yuchen {He}, Mike {Heaton}, Johannes {Heidecke}, Chris {Hesse}, Alan {Hickey}, Wade {Hickey}, Peter {Hoeschele}, Brandon {Houghton}, Kenny {Hsu}, Shengli {Hu}, Xin {Hu}, Joost {Huizinga}, Shantanu {Jain}, Shawn {Jain}, Joanne {Jang}, Angela {Jiang}, Roger {Jiang}, Haozhun {Jin}, Denny {Jin}, Shino {Jomoto}, Billie {Jonn}, Heewoo {Jun}, Tomer {Kaftan}, {\L}ukasz {Kaiser}, Ali {Kamali}, Ingmar {Kanitscheider}, Nitish {Shirish Keskar}, Tabarak {Khan}, Logan {Kilpatrick}, Jong~Wook {Kim}, Christina {Kim}, Yongjik {Kim}, Jan {Hendrik Kirchner}, Jamie {Kiros}, Matt {Knight}, Daniel {Kokotajlo}, {\L}ukasz
  {Kondraciuk}, Andrew {Kondrich}, Aris {Konstantinidis}, Kyle {Kosic}, Gretchen {Krueger}, Vishal {Kuo}, Michael {Lampe}, Ikai {Lan}, Teddy {Lee}, Jan {Leike}, Jade {Leung}, Daniel {Levy}, Chak~Ming {Li}, Rachel {Lim}, Molly {Lin}, Stephanie {Lin}, Mateusz {Litwin}, Theresa {Lopez}, Ryan {Lowe}, Patricia {Lue}, Anna {Makanju}, Kim {Malfacini}, Sam {Manning}, Todor {Markov}, Yaniv {Markovski}, Bianca {Martin}, Katie {Mayer}, Andrew {Mayne}, Bob {McGrew}, Scott~Mayer {McKinney}, Christine {McLeavey}, Paul {McMillan}, Jake {McNeil}, David {Medina}, Aalok {Mehta}, Jacob {Menick}, Luke {Metz}, Andrey {Mishchenko}, Pamela {Mishkin}, Vinnie {Monaco}, Evan {Morikawa}, Daniel {Mossing}, Tong {Mu}, Mira {Murati}, Oleg {Murk}, David {M{\'e}ly}, Ashvin {Nair}, Reiichiro {Nakano}, Rajeev {Nayak}, Arvind {Neelakantan}, Richard {Ngo}, Hyeonwoo {Noh}, Long {Ouyang}, Cullen {O'Keefe}, Jakub {Pachocki}, Alex {Paino}, Joe {Palermo}, Ashley {Pantuliano}, Giambattista {Parascandolo}, Joel {Parish}, Emy {Parparita}, Alex
  {Passos}, Mikhail {Pavlov}, Andrew {Peng}, Adam {Perelman}, Filipe {de Avila Belbute Peres}, Michael {Petrov}, Henrique {Ponde de Oliveira Pinto}, {Michael}, {Pokorny}, Michelle {Pokrass}, Vitchyr~H. {Pong}, Tolly {Powell}, Alethea {Power}, Boris {Power}, Elizabeth {Proehl}, Raul {Puri}, and Alec {Radford}.
\newblock {GPT-4 Technical Report}.
\newblock {\em arXiv e-prints}, page arXiv:2303.08774, March 2023.

\bibitem{Osband:TS:2018}
Ian Osband, John Aslanides, and Albin Cassirer.
\newblock Randomized prior functions for deep reinforcement learning.
\newblock In {\em Proceedings of the 32nd International Conference on Neural Information Processing Systems}, NIPS'18, page 8626–8638, Red Hook, NY, USA, 2018. Curran Associates Inc.

\bibitem{Osband:TS:2023}
Ian Osband, Zheng Wen, Seyed~Mohammad Asghari, Vikranth Dwaracherla, Morteza Ibrahimi, Xiuyuan Lu, and Benjamin Van~Roy.
\newblock Approximate thompson sampling via epistemic neural networks.
\newblock In {\em Proceedings of the Thirty-Ninth Conference on Uncertainty in Artificial Intelligence}, UAI '23. JMLR.org, 2023.

\bibitem{Ouyang:RLHF:2022}
Long Ouyang, Jeff Wu, Xu~Jiang, Diogo Almeida, Carroll~L. Wainwright, Pamela Mishkin, Chong Zhang, Sandhini Agarwal, Katarina Slama, Alex Ray, John Schulman, Jacob Hilton, Fraser Kelton, Luke Miller, Maddie Simens, Amanda Askell, Peter Welinder, Paul Christiano, Jan Leike, and Ryan Lowe.
\newblock Training language models to follow instructions with human feedback.
\newblock In {\em Proceedings of the 36th International Conference on Neural Information Processing Systems}, NIPS '22, Red Hook, NY, USA, 2022. Curran Associates Inc.

\bibitem{Ouyang:ChatGPT:2022}
Long {Ouyang}, Jeff {Wu}, Xu~{Jiang}, Diogo {Almeida}, Carroll~L. {Wainwright}, Pamela {Mishkin}, Chong {Zhang}, Sandhini {Agarwal}, Katarina {Slama}, Alex {Ray}, John {Schulman}, Jacob {Hilton}, Fraser {Kelton}, Luke {Miller}, Maddie {Simens}, Amanda {Askell}, Peter {Welinder}, Paul {Christiano}, Jan {Leike}, and Ryan {Lowe}.
\newblock {Training language models to follow instructions with human feedback}.
\newblock {\em arXiv e-prints}, page arXiv:2203.02155, March 2022.

\bibitem{2023arXiv230518290R}
Rafael {Rafailov}, Archit {Sharma}, Eric {Mitchell}, Stefano {Ermon}, Christopher~D. {Manning}, and Chelsea {Finn}.
\newblock {Direct Preference Optimization: Your Language Model is Secretly a Reward Model}.
\newblock {\em arXiv e-prints}, page arXiv:2305.18290, May 2023.

\bibitem{Russo:TS:2019}
Daniel Russo.
\newblock Worst-case regret bounds for exploration via randomized value functions.
\newblock {\em ArXiv}, abs/1906.02870, 2019.

\bibitem{Russo:bandit:2014}
Daniel Russo and Benjamin Van~Roy.
\newblock Learning to optimize via posterior sampling.
\newblock {\em Math. Oper. Res.}, 39(4):1221–1243, November 2014.

\bibitem{Saha:bandit:2021}
Aadirupa Saha.
\newblock Optimal algorithms for stochastic contextual preference bandits.
\newblock In M.~Ranzato, A.~Beygelzimer, Y.~Dauphin, P.S. Liang, and J.~Wortman Vaughan, editors, {\em Advances in Neural Information Processing Systems}, volume~34, pages 30050--30062. Curran Associates, Inc., 2021.

\bibitem{Saha:bandit:2022}
Aadirupa Saha and Akshay Krishnamurthy.
\newblock Efficient and optimal algorithms for contextual dueling bandits under realizability.
\newblock In Sanjoy Dasgupta and Nika Haghtalab, editors, {\em Proceedings of The 33rd International Conference on Algorithmic Learning Theory}, volume 167 of {\em Proceedings of Machine Learning Research}, pages 968--994. PMLR, 29 Mar--01 Apr 2022.

\bibitem{Saha:RLHF:2023}
Aadirupa Saha, Aldo Pacchiano, and Jonathan Lee.
\newblock Dueling rl: Reinforcement learning with trajectory preferences.
\newblock In Francisco Ruiz, Jennifer Dy, and Jan-Willem van~de Meent, editors, {\em Proceedings of The 26th International Conference on Artificial Intelligence and Statistics}, volume 206 of {\em Proceedings of Machine Learning Research}, pages 6263--6289. PMLR, 25--27 Apr 2023.

\bibitem{Schulman:PPO:2017}
John {Schulman}, Filip {Wolski}, Prafulla {Dhariwal}, Alec {Radford}, and Oleg {Klimov}.
\newblock {Proximal Policy Optimization Algorithms}.
\newblock {\em arXiv e-prints}, page arXiv:1707.06347, July 2017.

\bibitem{Stiennon:RLHF:2020}
Nisan Stiennon, Long Ouyang, Jeff Wu, Daniel~M. Ziegler, Ryan Lowe, Chelsea Voss, Alec Radford, Dario Amodei, and Paul Christiano.
\newblock Learning to summarize from human feedback.
\newblock In {\em Proceedings of the 34th International Conference on Neural Information Processing Systems}, NIPS '20, Red Hook, NY, USA, 2020. Curran Associates Inc.

\bibitem{Thompson:TS:1933}
William~R. Thompson.
\newblock On the likelihood that one unknown probability exceeds another in view of the evidence of two samples.
\newblock {\em Biometrika}, 25(3/4):285--294, 1933.

\bibitem{Touvron:2023}
Hugo {Touvron}, Louis {Martin}, Kevin {Stone}, Peter {Albert}, Amjad {Almahairi}, Yasmine {Babaei}, Nikolay {Bashlykov}, Soumya {Batra}, Prajjwal {Bhargava}, Shruti {Bhosale}, Dan {Bikel}, Lukas {Blecher}, Cristian {Canton Ferrer}, Moya {Chen}, Guillem {Cucurull}, David {Esiobu}, Jude {Fernandes}, Jeremy {Fu}, Wenyin {Fu}, Brian {Fuller}, Cynthia {Gao}, Vedanuj {Goswami}, Naman {Goyal}, Anthony {Hartshorn}, Saghar {Hosseini}, Rui {Hou}, Hakan {Inan}, Marcin {Kardas}, Viktor {Kerkez}, Madian {Khabsa}, Isabel {Kloumann}, Artem {Korenev}, Punit {Singh Koura}, Marie-Anne {Lachaux}, Thibaut {Lavril}, Jenya {Lee}, Diana {Liskovich}, Yinghai {Lu}, Yuning {Mao}, Xavier {Martinet}, Todor {Mihaylov}, Pushkar {Mishra}, Igor {Molybog}, Yixin {Nie}, Andrew {Poulton}, Jeremy {Reizenstein}, Rashi {Rungta}, Kalyan {Saladi}, Alan {Schelten}, Ruan {Silva}, Eric~Michael {Smith}, Ranjan {Subramanian}, Xiaoqing~Ellen {Tan}, Binh {Tang}, Ross {Taylor}, Adina {Williams}, Jian~Xiang {Kuan}, Puxin {Xu}, Zheng {Yan}, Iliyan {Zarov},
  Yuchen {Zhang}, Angela {Fan}, Melanie {Kambadur}, Sharan {Narang}, Aurelien {Rodriguez}, Robert {Stojnic}, Sergey {Edunov}, and Thomas {Scialom}.
\newblock {Llama 2: Open Foundation and Fine-Tuned Chat Models}.
\newblock {\em arXiv e-prints}, page arXiv:2307.09288, July 2023.

\bibitem{Geer2000EmpiricalPI}
Sara~A. van~de Geer.
\newblock Empirical processes in m-estimation.
\newblock 2000.

\bibitem{2023arXiv230916240W}
Chaoqi {Wang}, Yibo {Jiang}, Chenghao {Yang}, Han {Liu}, and Yuxin {Chen}.
\newblock {Beyond Reverse KL: Generalizing Direct Preference Optimization with Diverse Divergence Constraints}.
\newblock {\em arXiv e-prints}, page arXiv:2309.16240, September 2023.

\bibitem{SiweiWang:TS:2018}
Siwei Wang and Wei Chen.
\newblock Thompson sampling for combinatorial semi-bandits.
\newblock In Jennifer Dy and Andreas Krause, editors, {\em Proceedings of the 35th International Conference on Machine Learning}, volume~80 of {\em Proceedings of Machine Learning Research}, pages 5114--5122. PMLR, 10--15 Jul 2018.

\bibitem{Yuanhao:RLHF:2023}
Yuanhao Wang, Qinghua Liu, and Chi Jin.
\newblock Is rlhf more difficult than standard rl? a theoretical perspective.
\newblock In {\em Proceedings of the 37th International Conference on Neural Information Processing Systems}, NIPS '23, Red Hook, NY, USA, 2023. Curran Associates Inc.

\bibitem{Wirth:RLHF:2017}
Christian Wirth, Riad Akrour, Gerhard Neumann, and Johannes F{{\"u}}rnkranz.
\newblock A survey of preference-based reinforcement learning methods.
\newblock {\em Journal of Machine Learning Research}, 18(136):1--46, 2017.

\bibitem{Runzhe:RLHF:2023}
Runzhe {Wu} and Wen {Sun}.
\newblock {Making RL with Preference-based Feedback Efficient via Randomization}.
\newblock {\em arXiv e-prints}, page arXiv:2310.14554, October 2023.

\bibitem{Yue:bandit:2024}
Yue Wu, Tao Jin, Qiwei Di, Hao Lou, Farzad Farnoud, and Quanquan Gu.
\newblock Borda regret minimization for generalized linear dueling bandits.
\newblock In {\em Proceedings of the 41st International Conference on Machine Learning}, ICML'24. JMLR.org, 2024.

\bibitem{Xiong:RLHF:2023}
Wei {Xiong}, Hanze {Dong}, Chenlu {Ye}, Ziqi {Wang}, Han {Zhong}, Heng {Ji}, Nan {Jiang}, and Tong {Zhang}.
\newblock {Iterative Preference Learning from Human Feedback: Bridging Theory and Practice for RLHF under KL-Constraint}.
\newblock {\em arXiv e-prints}, page arXiv:2312.11456, December 2023.

\bibitem{ZhihanXiong:TS:2022}
Zhihan Xiong, Ruoqi Shen, Qiwen Cui, Maryam Fazel, and Simon~S Du.
\newblock Near-optimal randomized exploration for tabular markov decision processes.
\newblock In S.~Koyejo, S.~Mohamed, A.~Agarwal, D.~Belgrave, K.~Cho, and A.~Oh, editors, {\em Advances in Neural Information Processing Systems}, volume~35, pages 6358--6371. Curran Associates, Inc., 2022.

\bibitem{Yichong:RLHF:2020}
Yichong Xu, Ruosong Wang, Lin~F. Yang, Aarti Singh, and Artur Dubrawski.
\newblock Preference-based reinforcement learning with finite-time guarantees.
\newblock In {\em Proceedings of the 34th International Conference on Neural Information Processing Systems}, NIPS '20, Red Hook, NY, USA, 2020. Curran Associates Inc.

\bibitem{2023arXiv230405302Y}
Zheng {Yuan}, Hongyi {Yuan}, Chuanqi {Tan}, Wei {Wang}, Songfang {Huang}, and Fei {Huang}.
\newblock {RRHF: Rank Responses to Align Language Models with Human Feedback without tears}.
\newblock {\em arXiv e-prints}, page arXiv:2304.05302, April 2023.

\bibitem{Yue:bandit:2012}
Yisong Yue, Josef Broder, Robert Kleinberg, and Thorsten Joachims.
\newblock The k-armed dueling bandits problem.
\newblock {\em Journal of Computer and System Sciences}, 78(5):1538--1556, 2012.
\newblock JCSS Special Issue: Cloud Computing 2011.

\bibitem{Zanette:TS:2019}
Andrea {Zanette}, David {Brandfonbrener}, Emma {Brunskill}, Matteo {Pirotta}, and Alessandro {Lazaric}.
\newblock {Frequentist Regret Bounds for Randomized Least-Squares Value Iteration}.
\newblock {\em arXiv e-prints}, page arXiv:1911.00567, November 2019.

\bibitem{Wenhao:RLHF:2023}
Wenhao {Zhan}, Masatoshi {Uehara}, Nathan {Kallus}, Jason~D. {Lee}, and Wen {Sun}.
\newblock {Provable Offline Preference-Based Reinforcement Learning}.
\newblock {\em arXiv e-prints}, page arXiv:2305.14816, May 2023.

\bibitem{Wenhao_Zhan:RLHF:2024}
Wenhao Zhan, Masatoshi Uehara, Nathan Kallus, Jason~D. Lee, and Wen Sun.
\newblock Provable offline preference-based reinforcement learning.
\newblock In {\em The Twelfth International Conference on Learning Representations}, 2024.

\bibitem{2023arXiv230510425Z}
Yao {Zhao}, Rishabh {Joshi}, Tianqi {Liu}, Misha {Khalman}, Mohammad {Saleh}, and Peter~J. {Liu}.
\newblock {SLiC-HF: Sequence Likelihood Calibration with Human Feedback}.
\newblock {\em arXiv e-prints}, page arXiv:2305.10425, May 2023.

\bibitem{2023arXiv230602231Z}
Banghua {Zhu}, Hiteshi {Sharma}, Felipe {Vieira Frujeri}, Shi {Dong}, Chenguang {Zhu}, Michael~I. {Jordan}, and Jiantao {Jiao}.
\newblock {Fine-Tuning Language Models with Advantage-Induced Policy Alignment}.
\newblock {\em arXiv e-prints}, page arXiv:2306.02231, June 2023.

\bibitem{YinglunZhu:RL:NIPS:2022}
Yinglun Zhu and Robert Nowak.
\newblock Efficient active learning with abstention.
\newblock In {\em Proceedings of the 36th International Conference on Neural Information Processing Systems}, NIPS '22, Red Hook, NY, USA, 2022. Curran Associates Inc.

\end{thebibliography}
\bibliographystyle{plain}

\if{0}
\newpage

\appendix

\fi

\end{document}